\documentclass[a4paper,twoside,12pt,english]{article}

\usepackage[applemac]{inputenc}
\usepackage[english]{babel}
\makeatletter
\renewcommand\section{\@startsection
{section}{1}{0mm}%
{-2\bigskipamount}%
{\bigskipamount}%
{\normalfont\normalsize\bfseries}%
}
\renewcommand\thesection{\arabic{section}}

\newcommand\dateymd{\number\year, \ifcase\month\or
January\or February\or March\or April\or May\or June\or
July\or August\or September\or October\or November\or
December\fi, \number\day}
\newcommand\printtime{%
\c@hours=\time \divide\c@hours by60
\c@minutes=\c@hours \multiply\c@minutes by-60
\advance \c@minutes by \time
\ifnum\c@hours<10 0\fi\the\c@hours:%
\ifnum\c@minutes<10 0\fi\the\c@minutes}
\makeatother

\usepackage[OT1]{fontenc}

\usepackage{fancyhdr}
\usepackage{relsize}
\usepackage{bm} 
\usepackage{amssymb}
\usepackage{eucal}
\usepackage{amsmath}
\usepackage{amsthm}
\usepackage{enumerate}
\usepackage{IEEEtrantools}
\usepackage{colortbl}

\usepackage{verbatim} 

\newcommand\bnou{\textup{\textbf{\lower3.7pt\hbox{\char'052}:~}}}
\newcommand\enou{\unskip\textup{\textbf{~:\lower3.7pt\hbox{\char'052}}} }
\newcommand\bvell{\textup{\textbf{\lower3.7pt\hbox{\char'052}:~}$\langle$}}
\newcommand\evell{\unskip\textup{$\rangle$\textbf{~:\lower3.7pt\hbox{\char'052}}} }
\newcommand\bbnou{\textup{\textbf{\lower3.7pt\hbox{\char'052\char'052}:~}}}
\newcommand\eenou{\unskip\textup{\textbf{~:\lower3.7pt\hbox{\char'052\char'052}}} }

\newcommand\ie{i.\,e.~}

\newcommand\ifoi{\,\hbox{if\kern2.5pt and\kern2.5pt only\kern2.5pt if}\,{} }

\newcommand\df{\bfseries}
\newcommand\dfc[1]{\,{\df#1}\,}
\newcommand\dfd[1]{\,{\df#1}\hskip1pt}
\newcommand\secpar[1]{\S\,{#1}}

\newcommand\ensep{\unskip\hskip.65em\ignorespaces}

\newcommand\atilde{\lower3.5pt\hbox{\~{}}}
\newcommand\underl{\lower3.5pt\hbox{-}}

\newcommand\halfsmallskip{\vskip0.5\smallskipamount}

\newcommand\remark{\vskip-10pt\bigskip\noindent\textit{Remark}.\hskip.5em}

\newcommand\pq[2]{\raise.25ex\hbox{\footnotesize${#1}\over{#2}$}%
\hskip-.35ex\null}
\newcommand\onehalf{\frac12}

\newcommand\itm{\smallskip\noindent\hbox to\parindent{\hss\smaller{$\bullet$}\hskip.5em}}

\newcommand\xxxx[1]{%
 \hangindent2.5\parindent
 \hangafter1
 \noindent\hskip.5\parindent
 \hbox to2\parindent{\hss#1\hss}}
\newcommand\condition[2]{\xxxx{#1}\textit{#2}.}
\newcommand\iim[1]{\xxxx{\textup{(#1)}}\ignorespaces} 

\newcommand\iimtxt[1]{\textup{\small(#1)}~\ignorespaces}

\newcommand\ddd[1]{\halfsmallskip\vskip-2pt\noindent\hbox to 2\parindent{\hss\footnotesize$\bullet$\ \ }{#1}\ensep}

\newcommand\diaglabel[1]{\cellcolor[gray]{0.8}\makebox[1.5em][c]{$#1$}}
\newcommand\hlinestrut{\hline\rule{0pt}{2.4ex}}

\newcommand\sbset{\subset}

\newcommand\sbseteq{\subseteq}
\newcommand\spseteq{\supseteq}

\newcommand\cd[1]{\!#1\!}
\newcommand\cdsucc{\!\succ\!}

\newtheorem{proposition}{Proposition}[section]
\newtheorem{lemma}[proposition]{Lemma}
\newtheorem{theorem}[proposition]{Theorem}
\newtheorem{corollary}[proposition]{Corollary}

\newcommand\thesis{t}

\newcommand\nt[1]{\overline{#1}}		

\newcommand\sprm{s}
\newcommand\temp{t}
\newcommand\good{g}

\newcommand\maxcol{\tau} 			
\newcommand\minrow{\sigma} 		
\newcommand\minrw{\sigma}
\newcommand\mincol{\rho}

\newcommand\plu{f}
\newcommand\aplu{\bar{f}}

\newcommand\piset{\varPi}			
\newcommand\lit{p}		    		
\newcommand\liit{q}		    		
\newcommand\lxt{\alpha}		    
\newcommand\clau{C}				
\newcommand\doct{{\mathcal D}} 	
\newcommand\cnf{\Phi} 			

\newcommand\ist{A}				
\newcommand\xst{X}				
\newcommand\yst{Y}				

\newcommand\xstbis{\hbox to1.97ex{\hss\hskip2.5pt$\smash{\widetilde{%
 \hbox to1.9ex{\hss\vphantom{t}\smash{$\xst$}\hskip2.5pt\hss}}}$\hss}}
\newcommand\xstbiss{\smash{\widetilde\xst}}
\newcommand\xsth{\smash{\widehat X}}
\newcommand\rect{R}

\newcommand\val{w}				
\newcommand\orv{v}				
\newcommand\utv{v'}				
\newcommand\urv{v^*}				
\newcommand\orvbis{w}				
\newcommand\utvbis{w'}			
\newcommand\urvbis{w^*}			

\newcommand\Orv{V}				
\newcommand\Utv{V'}				
\newcommand\Urv{V^*}				

\newcommand\valtz{\widetilde w{}^{\kern.5pt\prime}}

\newcommand\mg{\eta}				

\newcommand\orvz{\widetilde v}	
\newcommand\utvz{\widetilde v{}^{\kern.5pt\prime}}
\newcommand\urvz{\widetilde v{}^{\kern.5pt\ast}}

\newcommand\vk{v^k}

\newcommand\ntv[1]{v^{(#1)}}		
\newcommand\nmaxcol[1]{\tau^{(#1)}}
\newcommand\nminrow[1]{\sigma^{(#1)}}
\newcommand\nminrw[1]{\sigma^{(#1)}}

\newcommand\better{\!\succ\!}

\newcommand\fiav{\!\mid\!}

\DeclareMathOperator*{\Max}{Max}

\newcommand\goodset{M} 				
\newcommand\smithset{S}				
\newcommand\smithsetbis{T}
\newcommand\majap{A^*} 				
\newcommand\majapgoodset{M^*} 	

\newcommand\res[1]{\,\vtop{\offinterlineskip\halign{\hfil##\hfil\cr$\vee$\cr\noalign{\vskip2pt}$\scriptstyle#1$\cr}}\,}

\renewcommand\res[1]{\mathbin{\vtop{\baselineskip0pt\lineskip.3ex\halign{\hfil##\hfil\cr$\vee$\cr$\scriptstyle{#1}$\cr}}}}

\newcommand\bla{}

\newlength\repskip 
{\null\hskip-\repskip\hbox to0.9\hsize\bgroup\hbox to0pt{\small$(\ref{#1})$\hss}\hfil\hskip.1\hsize$\displaystyle}%
{$\hfil\egroup}

{\hbox to\hsize\bgroup\hbox to0pt{\small$(\ref{#1})$\hss}\hfil$\displaystyle}%
{$\hfil\egroup\nonumber}

\begin{document}


\thispagestyle{empty}

\null\vskip-28mm\null 

\begin{center}
\hrule
\vskip7.5mm
\textbf{\uppercase{Social choice rules}}\par
\textbf{\uppercase{driven by propositional logic}\,%
\footnote{The original version of this article had the title
``Choosing and ranking. Let's be logical about it.''}}%
\par\medskip
\textsc{Rosa Camps,\, Xavier Mora \textup{and} Laia Saumell}
\par
Departament de Matem\`{a}tiques,\break
Universitat Aut\`onoma de Barcelona,\break
Catalonia
\par\medskip
\texttt{xmora\,@\,mat.uab.cat}
\par\medskip
July 28, 2011$^1$;\ensep revised September 30, 2013
\vskip5mm
\hrule
\end{center}

\begin{abstract}
Several rules for social choice
are examined from a unifying point of view
that looks at them as procedures for revising a system of degrees of  belief
in accordance with certain specified logical constraints.
Belief is here a~social attribute,
its degrees being measured by the fraction of people who share a~given opinion.
Different known rules and some new ones are obtained
depending on which particular constraints are assumed. 
These constraints allow to model different notions of choiceness.
In particular, we give a new method to deal with approval-disapproval-preferential voting.


\bigskip\noindent
\textbf{Classification MSC2010:}\hskip.75em
\textit{%
03B42, 
91B06, 
91B14. 
}
\end{abstract}

\section{Introduction}

\medskip
In this article we develop certain applications of a method for revising degrees of belief that we introduced in~\cite{dp}.
These applications belong to social choice theory.\linebreak[3] As general references about the latter, we refer the reader to~\cite{mu,nitzan,t6}.

\paragraph{1.1}
As its name says, the main subject matter of social choice theory is choosing among several options
in accordance with the existing preferences about them.
It~is~indeed a question of aggregating a set of individual opinions so as to define a collective one.
The~individual opinions will usually have an all-or-none character.
For instance, given two options~$x$ and $y$, either $x$ is preferred to $y$ or viceversa.
Or,~given a single option~$x$, either it is considered a right choice or it is not.
Putting together several opinions of this kind results in a more quantitative sort of information:\linebreak[3]
every particular statement, such as `$x$ is preferable to $y$', or `$x$ is a right choice',
is~now valued by the fraction of people who have expressed this view.
This fraction can be assimilated to a degree of collective belief in that statement.

If a collective decision must be adopted, it would be reasonable to abide by the majority,
\ie to accept a statement whenever that fraction is larger than one half.


But it is not so simple. Preferences are usually assumed to be transitive;
besides, it is taken for granted that an option deserves being chosen if and only if
it is preferred to any other.
However, it is well known
(see for instance \cite[\secpar7.1]{nitzan} and \cite[ch.\,9]{t6})
that these standard assumptions cannot be maintained when preferences are aggregated
and one tries to decide by means of the majority criterion.
Consider, for instance, three options $a,b,c$ and $15$~voters ---or 15~millions of them---
who rank these options in the following way:
\begin{equation}
\label{eq:ex1}
6: a\better b\better c,\quad 5: b\better c\better a,\quad 4: c\better a\better b.
\end{equation}
The number that precedes each of these rankings indicates how many people expressed it.
From this information one sees that $a$ is preferred to $b$ by a majority of people, namely 10 against 5. Similarly, $b$ is preferred to $c$ by a majority of 11 against 4, and $c$ is preferred to $a$ by a majority of 9 against 6.
These numbers are collected in the following table, that we call the Llull matrix of the vote:%
\footnote{Since we are interested only in the preferences of $x$ over $y$ for $x\neq y$, we use the diagonal cells for specifying the simultaneous labelling of rows and columns by the existing options. The cell located in row $x$ and column $y$ gives information about the preference of $x$ over $y$.}
\begin{equation}
\label{eq:llull1}
\begin{tabular}{|c|c|c|}
\hlinestrut
\diaglabel{a} & \textbf{10} & 6\\
\hlinestrut
5 & \diaglabel{b} & \textbf{11}\\
\hlinestrut
\textbf{9} & 4 & \diaglabel{c}\\
\hline
\end{tabular}
\,.
\end{equation}
So the majoritarian preferences are not transitive; in fact, they form a cycle.
Besides, they do not produce an option with the property of being preferred to every other.

This is a particular case of the general problem of judgment aggregation \cite{list,listpuppe},
where a~group of people wants to decide on several issues that are subject to certain logical constraints.
Even when each individual gives an opinion that is consistent with these logical constraints,
the aggregate opinion defined by the majority criterion can lose such a consistency.
This poses the problem of which method should be used to arrive at a consistent decision.

We dealt with this problem in~\cite{dp}.
Our method, that will be summarized in~\secpar{2},
hinges on a clear statement of the logical constraints that
relate the issues in question to each other. 
Every issue is represented by an atomic proposition
and every constraint is represented by a compound proposition whose truth is assumed to hold.

In social choice theory, one is interested in the propositions $p_{xy}$: `$x$ is preferable to $y$',\linebreak[3] and $q_x$: `$x$ is a right choice', where $x$ and $y$ vary over the set of options. The standard notion of preference assumes the constraints of antisymmetry, namely $\nt p_{xy} \leftrightarrow p_{yx}$ for any $x,y$ different from each other, and transitivity, namely $p_{xy} \land p_{yz} \rightarrow p_{xz}$ for any $x,y,z$ pairwise different from each other. However, in most of this article we will drop the constraint of transitivity and we will keep only that of antisymmetry. The~reason for it is that the notion of choiceness, \ie being a right choice,
is not really concerned with transitivity. In fact, social choice theory often
distinguishes between choice functions and ranking functions (see for instance \cite{young86} and \cite{bala}).


Even if we equate choiceness to being preferred to anything else,
this corresponds simply to the constraint $q_x \leftrightarrow \bigwedge_{y\neq x} p_{xy}.$
We will refer to this notion of choiceness as \dfc{supremacy.}

Besides it, we will consider also other notions.
For instance, instead of the preceding double-implication constraint,
it~makes sense to require only the single implications
$\bigwedge_{y\neq x} p_{xy} \rightarrow q_x$ 
and $\bigwedge_{y\neq x} p_{yx} \rightarrow \nt q_x.$
The first of these is related to the classical principle of Llull and Condorcet,
namely that an option should be chosen if it is preferred to every other by a majority.
These implications define an alternative notion of choiceness that we call \dfc{prominence.}


Properly speaking, however, having definite binary preferences between options
does not preclude the possibility that none of them is considered good enough,
or that all of them are considered so.
This is very much the idea of approval voting~\cite{brams}.
Of course, if $x$ is considered good and $y$ is preferred to $x$ 
then $y$ should also be considered good,
that is, we are constrained to satisfy the implication $q_x \land p_{yx} \rightarrow q_y$.
This notion of choiceness will be referred to as \dfc{goodness.}

Quite interestingly, Condorcet himself advocated for such a relaxation of the notion of choiceness.
In 1789, four years after his celebrated \textit{Essai} where he considered the above-mentioned principle, he expressed himself in the following way:
``It is generally more important to be sure of electing men who are worthy of holding office than to have a small probability of electing the worthiest man'' \cite{condorcet}.

Let us remark here that later on the notation $q_x$ will be replaced by different symbols
---namely, $\sprm_x, \temp_x, \good_x$---
depending on which notion of choiceness is being considered.

\paragraph{1.2}
Assume that
several individuals are asked not only to rank a set of options, but also to indicate which of them meet their approval. Of course, each individual is supposed to be consistent at putting his approved options at the top of his ranking. Once again, however, the majority criterion can produce an inconsistent result. Assume, for instance, that the votes are as follows:
\begin{equation}
\label{eq:ex11}
5: a\better b\better c\fiav\,,\quad
4: b\better a\better c\fiav\,,\quad
8: b\fiav a\better c,\quad
9: c\fiav a\better b,
\end{equation}
where a vertical bar indicates that the options at the left of it are approved whereas those at the right are disapproved. By counting the preferences about every pair of options, we get the following Llull matrix: 
\begin{equation}
\label{eq:llull11}
\begin{tabular}{|c|c|c|}
\hlinestrut
\diaglabel{a} & \textbf{14} & \textbf{17}\\
\hlinestrut
12 & \diaglabel{b} & \textbf{17}\\
\hlinestrut
9 & 9 & \diaglabel{c}\\
\hline
\end{tabular}
\,.
\end{equation}
In contrast to (\ref{eq:llull1}), here the majority criterion
results in a complete ordering, namely $a\better b\better c$.
However, if we count the number of approvals and disapprovals that are present in (\ref{eq:ex11}),
we see that option~$a$ is disapproved by a majority, namely by 17~individuals,
whereas $b$ is approved by a majority consisting also of 17~individuals
and $c$~is approved by a majority of 18~individuals.
So, the option that goes first in the obtained ranking
is the most disapproved one!
Which option should be chosen in such a situation?


\paragraph{1.3}
As we have already mentioned, we will be looking at those numbers of people
as degrees of (collective) belief,
and we will revise them in the light of the assumed constraints. 
Our revision method, that was introduced in \cite{dp}, is
based on the following general principle:
an implication of the form
$(\lit_1 \land \lit_2 \land \dots \land \lit_n) \rightarrow \liit$
with a satisfiable left-hand side gives to its conclusion~$\liit$
at least the same degree of belief
as the weakest of the premises~$\lit_i$.

This principle has been used by several authors
(see especially \cite{rescher76})
and it can be traced back to ancient philosophy,
where it was stated by saying that \textit{peiorem semper conclusio sequitur partem}.

Let us illustrate our usage of this principle in the case of (\ref{eq:ex1}),
where we saw that the collective preferences defined by the majority criterion are not consistent with transitivity.

For each pair of options $x$ and $y$, the preference $p_{xy}$ is supported by the number of people indicated in the table~(\ref{eq:llull1}).
We view these numbers as degrees of collective belief.
From a theoretical point of view, it is natural to normalize them to the interval $[0,1]$
by considering their ratio to the total number of voters $V,$ 15~in this case.
However, in practical cases like this one it is more convenient to stay with the numbers of people,
since they are small integers.
In the following we denote these numbers by $V(p_{xy})$ and $v(p_{xy})$, the latter being the normalized ones, that is $v(p_{xy}) = V(p_{xy}) / V.$

The lack of consistency with transitivity occurs in the following way:
$a$ is collectively preferred to $b$ since $V(p_{ab}) = 10 > 5 = V(p_{ba}),$ 
$b$ is collectively preferred to $c$ since $V(p_{bc}) = 11 > 4 = V(p_{cb}),$
but $c$ is collectively preferred to $a$ since $V(p_{ca}) = 9 > 6 = V(p_{ac}).$
This is not consistent with the implication that defines transitivity,
namely $p_{ab} \land p_{bc} \rightarrow p_{ac}$.
However, if we assume this implication to be true,
then the above-mentioned principle of the weakest premise
allows to increase the degree of belief in $p_{ac}$ to the value
$V^*(p_{ac}) = \min(V(p_{ab}), V(p_{bc})) = 10$.

By proceeding in this way with all triads of options,
we arrive at the following revised degrees of belief:
\begin{equation}
\label{eq:vp1}
(V^*(p_{xy})) \,=\,
\begin{tabular}{|c|c|c|}
\hlinestrut
\diaglabel{a} & \textbf{10} & \textbf{10}\\
\hlinestrut
9 & \diaglabel{b} & \textbf{11}\\
\hlinestrut
9 & 9 & \diaglabel{c}\\
\hline
\end{tabular}
\,.
\end{equation}
One can easily check that these new degrees of belief are not increased
by applying the same procedure again.
However, for a larger number of options one would be led
to a repeated application of the same procedure for all triads of options,
or equivalently, to a similar procedure involving longer implications
such as $p_{ab} \land p_{bc} \land p_{cd} \rightarrow p_{ad}$.
Anyway, we will see that the final degrees of belief are always consistent with transitivity.
More specifically, transitivity is ensured if one redefines the collective preferences by considering $x$ preferred to $y$ whenever $V^*(p_{xy}) > V^*(p_{yx})$.



Although it has to do with other constraints, the inconsistency that we have seen to arise from (\ref{eq:ex11}) can also be resolved by a similar procedure
(that happens to choose neither $a$ nor $c$, but $b$!).

In fact, in the next section we will see that this procedure can be extended to quite general logical constraints and that it enjoys several desirable properties.


In subsequent sections we will apply it to the different notions of choiceness
that have been pointed out above.

As we will see, depending on which particular constraints are adopted, as well as other details,
one~obtains a variety of rules for social choice.
These rules will include some well-known ones, such as plurality, maximin,
Schulze's method of paths and approval voting.
On the other hand, we will also obtain some new rules,
such as a new Condorcet rule that we call the ``comprehensive prominence method''
and a new method for dealing with approval-disapproval-preferential voting
(the reader specifically interested in this particular application can skip sections 4--6).

Anyway, our method provides a common framework that reveals the precise logic
behind each of these rules.

\section{General framework}

In this section we summarize the general method given in \cite{dp}.
Its aim is to revise the existing degrees of belief
about several logically constrained issues
and to arrive at consistent decisions about them.

\paragraph{2.1}
The~issues under consideration are represented by a finite set of basic logical propositions
together with the corresponding negations.
This set of propositions \hbox{---their} negations included--- will be denoted as $\piset$,
and the negation of~$p$ will be denoted as~$\nt p$.
The elements of $\piset$ are referred to as \dfc{literals.}

\medskip
The logical constraints between issues
are referred to as
the \dfc{doctrine.}
They are specified by a set of compound propositions
that are required to be true. 
This entails a series of material implications between literals
that are conveniently codified 
by rewriting the set of those constraints in \dfc{conjunctive normal form,}
\ie in the form
\begin{equation}
\label{eq:cnf}
\cnf(\doct) \,:=\,
\bigwedge_{\clau\in\doct} \left(\,\bigvee_{\lit\in\clau} \lit\right),
\end{equation}
where $\doct$ stands for a certain collection of subsets of $\piset$.
Each expression within parentheses in the preceding formula
---or equivalently the corresponding set 
$\clau\sbset\piset$---\, is called a \dfc{clause.}

The conjunctive normal form of a doctrine is not unique.
Generally speaking, this can make a difference for the procedure that we are about to introduce.
This ambiguity is eliminated by resorting to the \dfc{Blake canonical form,}
that consists of all the prime clauses of the doctrine under consideration;
a clause being \dfc{prime} 
means that no proper subset of it is 
still entailed by the doctrine.
However, for many doctrines 
one is ensured to get the same results with other conjunctive normal forms made of prime clauses.
Such a form will be said to be \dfc{$\ast$-equivalent} to the Blake canonical one.
In particular, this happens whenever the form under consideration has a property that we call \dfc{disjoint-resolvability.}
\ensep
For these and other technical matters we refer the reader to \cite[\secpar{4}]{dp}.
\ensep
Anyway, the doctrine, that from now on we are assimilating to the set $\doct$, is required to satisfy the following conditions:
\ensep
\iimtxt{D1}
It~is satisfiable.
\ensep
\iimtxt{D2}
It does not contain unit clauses, \ie clauses with a single literal.
\ensep
\iimtxt{D3}
It explicitly contains the \textit{tertium non datur} clause $\lit\lor \nt\lit$ for any 
$\lit\in\piset$.
\iimtxt{D4} It is made of prime clauses.


\paragraph{2.2}
A~system of degrees of belief is represented by a~mapping $\val$ from $\piset$ to the interval $[0,1]$.
We refer to such a mapping as a~\dfc{valuation,} and 
the image of $p\in\piset$ by a particular valuation $\val$ will be denoted as $\val_p$ or $\val(p)$.
A~valuation $\val$ is called \dfc{balanced} when $\val_p +\val_{\nt p}$ is equal to $1$ for any $p\in\piset$.
The \dfc{truth assignments} of classical logic are balanced valuations with all-or-none values, that is either $0$~or~$1$.
In contrast,  \dfc{degrees of belief} can take fractional values;
besides, they need not be balanced: 
$\val_p + \val_{\nt p}$ may be less than~$1$ (lack of information)
or even greater than~$1$ (presence of contradiction).
In our approach, the latter case can arise because of the logical implications contained in the constraints.

We will also make use of \dfc{partial truth assignments.}
They will be seen as balanced valuations with values in~$\{0,\onehalf,1\}$,
where the value $\onehalf$ can be interpreted as `undefined' or `undecided'.
Partial truth assignments will be used mainly for specifying decisions,
on which case the values $1$, $0$ and $\onehalf$ can be interpreted as meaning respectively
`accepted', `rejected' and `undecided'.

A partial truth assignment $u$ will be said to be \dfc{definitely consistent} with~$\doct$,
or~with $\Phi(\doct)$, 
when, for each clause $\clau\in\doct$ and every $\lit\in\clau$,
the following implication holds:\ensep
if $u_\lxt=0$ for every $\lxt\in\clau\setminus\{\lit\}$,
then $u_\lit=1$.
When no undecidedness is present, definite consistency is equivalent to saying that
the truth assignment under consideration makes true the formula $\Phi(\doct)$.
When undecidedness is allowed, definite consistency requires every clause to contain at least one accepted literal, or alternatively, at least two undecided literals.

\medskip
Every valuation $\val$ gives rise to a (partial) decision in the following way,
that depends on a parameter $\mg$ in the interval $0\le\mg\le1$:
For any $\lit\in\piset$,
\begin{alignat}{4}
&\text{$\lit$ is \textbf{accepted} and $\nt\lit$ is \textbf{rejected}} &&\quad\text{whenever}\ \ 
\hphantom{|\,}\val_\lit &&- \val_{\nt\lit}\hphantom{\,|} &&\,>\, \mg,
\label{eq:acceptedrejected}
\\
&\text{$\lit$ and $\nt\lit$ are left \textbf{undecided}} &&\quad\text{whenever}\ \ 
|\,\val_\lit &&- \val_{\nt\lit}\,| &&\,\le\, \mg.
\label{eq:undecided}
\end{alignat}
We will refer to it as the \dfc{decision of margin $\mg$} associated with~$\val$, and
we will identify it with the corresponding partial truth assignment.
In the case $\mg=0$ we will call it the \dfc{basic decision} associated with $\val$.
\ensep
In tune with these definitions, the difference $\val_\lit-\val_{\nt\lit}$ will be called the \dfc{acceptability} of~$\lit$ according to $\val$.
\ensep
If the valuation $\val$ is balanced, then the basic decision criterion is equivalent to the majority rule, 
namely accepting $\lit$ and rejecting $\nt\lit$ whenever $\val_\lit>\onehalf$.

\paragraph{2.3}
A clause being true means that at least one of its literals is true;
in other words, if all of its literals but one are known to be false, then the remaining one must be true.
Therefore, the doctrine associated with (\ref{eq:cnf})
provides the following implications:
\begin{equation}
\label{eq:pimplicant}
\lit \,\leftarrow\, \bigwedge_{\substack{\lxt\in\clau\\\lxt\neq\lit}} \nt\lxt,
\end{equation}
for any $\clau\in\doct$ such that $\lit\in\clau$.
Each of these implications is a possible source of belief in~$\lit$.
In this connection, it makes sense to apply the classical rule that 
the conclusion $\lit$ should be believed at least as the weakest of the premises $\nt\lxt$.
This rule requires the right-hand side of (\ref{eq:pimplicant}) to be satisfiable, which is ensured because $\clau$ is prime.
This leads to the following procedure for revising any given degrees of belief $\orv$ about~$\piset$:
every $\lit\in\piset$ should be believed at least in the new degree $\utv_\lit$ defined by
\begin{equation}
\label{eq:vprime}
\utv_\lit \,=\, \max_{\substack{\clau\in\doct\\\clau\ni\lit}}\, \min_{\substack{\lxt\in\clau\\\lxt\neq\lit}} \,\orv_{\nt\lxt},
\end{equation}
One easily checks the truth of the following statement:

\renewcommand\bla{\cite[Lem.~3.1]{dp}}
\begin{lemma}[\bla]\hskip.5em
\label{st:step}
The transformation $\orv\mapsto\utv$ has the following properties:

\iim{a}It is continuous.

\iim{b}$\orv \le \orvbis$ implies $\utv \le \utvbis$.

\iim{c}$\orv \le \utv$.

\iim{d}The image set of $\utv$ is contained in that of $\orv$.
\end{lemma}

\medskip
As soon as we accept $\utv$ as new degrees of belief, it makes sense to repeat the same operation with $\orv$ replaced by $\utv$, thus obtaining a still higher valuation~$\orv''$, and so on. By proceeding in this way, one obtains a non-decreasing sequence of valuations $\ntv{n}\ (n = 0,1,2,\dots)$ with the property that all of them take values in the same finite set. Obviously, this implies that this sequence will eventually reach an invariant state~$\urv$. This eventual valuation is, by definition, the \dfd{upper revised valuation}.

\paragraph{2.4}
The main properties of the upper revised valuation are collected in the following statements:

\medskip
\renewcommand\bla{Basic facts \cite[Thm.~3.2]{dp}}
\begin{theorem}[\bla]\hskip.5em
\label{st:rev}
The transformation $\orv\mapsto\urv$ has the following properties:

\iim{a}It is continuous.

\iim{b}$\orv \le \orvbis$ implies $\urv \le \urvbis$.

\iim{c}$\orv \le \urv$.

\iim{d}The image set of $\urv$ is contained in that of $\orv$.
\end{theorem}

\renewcommand\bla{Characterization \cite[Thm.~3.3]{dp}}
\begin{theorem}[\bla]\hskip.5em
\label{st:char}
The upper revised valuation $\urv$ is the lowest of the valuations $\val$ that lie above~$\orv$ and 
are consistent with the doctrine in the sense of satisfying 
the equation $\val{}'=\val$.
\end{theorem}

\renewcommand\bla{Consistency of the associated decisions \cite[Cor.~3.7]{dp}}
\begin{theorem}[\bla]\hskip.5em
\label{st:dec}
For any $\mg$ in the interval $0\cd\le\mg\cd\le1$, 
the decision of margin~$\mg$ associated with 
the upper revised valuation is always definitely consistent
with the doctrine.
\end{theorem}

\renewcommand\bla{Respect for consistent majority decisions \cite[Thm.~3.9]{dp}}
\begin{theorem}[\bla]\hskip.5em
\label{st:majority}
Assume that every $p\in\piset$ satisfies either $\orv_\lit>\onehalf>\orv_{\nt\lit}$ or, contrarily, $\orv_{\nt\lit}>\onehalf>\orv_\lit$. Assume also that the basic decision associated with $\orv$ (which contains no undecidedness) is consistent with the doctrine. In this case, the basic decision associated with the upper revised valuation $\urv$ is the same.
\end{theorem}

\renewcommand\bla{Respect for unanimity \cite[Thm.~3.11]{dp}}
\begin{theorem}[\bla]\hskip.5em
\label{st:unanimity}
Assume that $\orv$ is an aggregate of consistent truth assignments.
In this case, having $\orv_\lit = 1$ implies that $\lit$ is accepted by the basic decision associated with the upper revised valuation~$\urv$.
\end{theorem}

\renewcommand\bla{Monotonicity \cite[Thm.~3.14 and Cor.~3.15]{dp}}
\begin{theorem}[\bla]\hskip.5em
\label{st:mono}
Assume that the valuation \,$\orv$\, is modified into a new one \,$\orvz$\, such that
\begin{equation}
\label{eq:mono}
\orvz_\lit \,>\, \orv_\lit,\qquad \orvz_\liit \,=\, \orv_\liit,\quad \forall q\in\piset\setminus\{\lit\}.
\end{equation}
In this case, the acceptability of $\lit$ either increases or stays constant:
\begin{equation}
\label{eq:mona}
\urvz_\lit - \urvz_{\nt\lit} \,\ge\, \urv_\lit - \urv_{\nt\lit}.
\end{equation}
As a consequence,
if $\lit$ is accepted \textup[resp.~not rejected\,\textup] in the decision of margin $\mg$ associated with~$\urv$,
then it is also accepted \textup[resp.~not rejected\,\textup] in~the decision of margin $\mg$ associated with~$\urvz$.
\end{theorem}


\paragraph{2.5}
Having the equality $\urv_\lit=\utv_\lit$ for the Blake canonical form ---or for any disjoint-resolvable prime conjunctive normal form--- guarantees that the degree of belief $\urv_\lit$ does not derive from unsatisfiable conjunctions \cite[\secpar{4.3}]{dp}. If~that equality holds no matter the initial valuation $\orv$, we say that the doctrine under consideration is \dfc{unquestionable for~$\lit$.} Sometimes, the equality can be guaranteed only under certain special circumstances. In~particular, it~can happen 
that it holds whenever $\lit$ is accepted according to~$\urv$; in that case we say that the doctrine is \dfc{unquestionable for~$\lit$ when accepted.}
Sufficient conditions for ensuring such properties are given in \cite[Thm.\,4.8, Cor.\,4.9]{dp}.

\paragraph{2.6}
The following fact will be useful for computations:

\medskip
\begin{lemma}\hskip.5em
\label{st:lem-best-option}
The successive valuations $\ntv{n}$ satisfy the following formula for any $n\ge1$:
\begin{equation}
\label{eq:vpn}
\ntv{n}_\lit \,=\, \max \big( \, \orv_\lit,
\max_{\substack{\clau\in\doct\\\clau\ni\lit\\\clau\neq\{\lit,\nt\lit\}}} \min_{\substack{\lxt\in\clau\\\lxt\neq\lit}} \,\ntv{n-1}_{\nt\lxt} \,\big).
\end{equation}

\end{lemma}

\begin{proof}\hskip.5em
By definition, $\ntv{n}$ is obtained from $\ntv{n-1}$
by the transformation $\orv\mapsto\utv$ defined by (\ref{eq:vprime}).
The resulting expression is the same as (\ref{eq:vpn}) except that the right-hand side shows $\ntv{n-1}_\lit$ instead of $\orv_\lit$.
In order to obtain formula (\ref{eq:vpn}) it suffices to apply repeatedly
the two following facts, which are easily checked by induction:
\ensep
(i)~If $a_n\ (n\ge0)$ satisfies $a_n = \max(a_{n-1}, b_{n-1})\ (n\ge1)$, 
where $b_n\ (n\ge0)$ is a non-decreasing sequence, 
then $a_n = \max(a_0, b_{n-1})$ for any $n\ge1$;
\ensep
(ii)~If $a_n$ and\, $b_n$ $(n\ge0)$ are non-decreasing sequences,
then the sequences $\max(a_n,b_n)$ and $\min(a_n,b_n)$ are also non-decreasing.
\end{proof}

%

\paragraph{2.7}
For our purposes, $\nt\lit$ need not be the exact semantic negation of $\lit$.
\ensep
Instead, quite often it is more appropriate to look at $\nt\lit$ as the opposite, or antithesis, of~$\lit$.
This may seem to conflict with the excluded-middle principle $\lit\lor\nt\lit$.
However, this principle somehow loses its character just as fractional valuations come in.
In~fact, its role in connection with the latter is only through the excluded-middle clauses that we systematically include in the Blake canonical form; and this has only the following two effects:
(a)~providing the trivial implications $\lit\rightarrow\lit$ and $\nt\lit\rightarrow\nt\lit$,
through which the revised degrees of belief become larger than or equal to the original ones;
and (b)~forbidding 
any implication of the form $\nt\lit\land\lit\land\chi\rightarrow\thesis$,
which would be a gratuitous source of belief
(this effect occurs because
clauses are restricted to be prime,
which prevents them from containing $\lit\lor\nt\lit$).

Anyway, the belief in $\nt\lit$ is not the lack of belief in $\lit$,
but it should have its own reasons.
This agrees with the general views of
\cite[see for instance p.\,12]{qu}.
Besides, it fully agrees also with the traditional views of the adversarial system of justice.





\medskip
In order to apply the preceding method to a particular matter, 
one must specify the main issues at stake
as well as the existing logical implications between them.
As we will see, the notions of preference and choiceness
allow for several views about which logical implications are associated with them.
As a result we will obtain several alternative models and rules for social choice.

\section{Preferences}

In the sequel we will be dealing all the time with a finite set of options. 
This set will be denoted as $\ist$.
A system of preferences about the members of~$\ist$ is~usually formalized as a binary relation on~$\ist$ that complies with certain properties,
typically including antisymmetry, completeness and transitivity
(see for instance \cite{hansson}).
\ensep
Having said that,
nowadays it is well established that preferences are sometimes not transitive
(see for instance \cite{tversky}, or \cite[\secpar{1.3}]{hansson}).
So, we will take the view that transitivity is a special way of having preferences about 
things.
\ensep
On the other hand, antisymmetry and completeness admit of certain alternatives 
depending on whether indifference and/or lack of opinion are allowed into consideration;
as we will see next, however, these alternatives become unnecessary when fractional and possibly unbalanced degrees of belief are used (recall that a lack of balance means that the degrees belief associated with $\lit$ and $\nt\lit$ need not add up to $1$).

\medskip
Describing preferences by means of propositional logic requires considering all of the propositions 
$p_{xy}$: `$x$ is preferable to $y$', where $x$ and $y$ are different from each other
(allowing for $x=y$ leads to useless distinctions).

From our point of view, the main principle associated with the notion of preference
is that $\nt p_{xy}$ can be identified with $p_{yx}$, or equivalently, that
\begin{equation}
\label{eq:iff}
\nt p_{xy} \,\leftrightarrow\, p_{yx}, \qquad \text{for any two different $x,y\in\ist$.}
\end{equation}


In the all-or-none framework of classical logic,
this double implication 
embodies a limitation to complete strict preferences,
leaving no place for 
definite indifference ($x$ and $y$ `equally good')
nor for incompleteness (lack of information about the comparison between $x$ and~$y$).

However, in our context of degrees of belief, 
both definite indifference and incompleteness can be suitably modelled
if we interpret (\ref{eq:iff}) as an identification between $\nt p_{xy}$ and $p_{yx}$.
In fact, incompleteness can be described by putting $v(p_{xy}) = v(p_{yx}) = 0$
which means a full lack of belief in $p_{xy}$ as well as in $p_{yx}$.
On~the other hand, definite indifference can be described by putting $v(p_{xy}) = v(p_{yx}) = \onehalf$,
\ie by splitting the unit of belief into equal amounts for the contrary preferences $p_{xy}$ and $p_{yx}$.

\bigskip
So, \,\textit{from now on we identify \,$\nt p_{xy}$ with \,$p_{yx}$.}

\bigskip
The collective degrees of belief about the propositions $p_{xy}$ are given by
\begin{equation}
\label{eq:cog}
v(p_{xy}) = \sum_k \alpha_k\, \vk(p_{xy}),
\end{equation}
where $\alpha_k$ are the relative frequencies or weights of the individual opinions~$\vk$.
\ensep
In preferential voting, the individual opinions are usually expressed in the form of a ranking,
that is, a list of options in order of preference, possibly truncated or with ties.
In order to translate this information into paired comparisons, we use the following interpretation:

\halfsmallskip
\iim{a}When $x$ and $y$ are both in the list\,
and $x$ is ranked above $y$ (without a tie),
we certainly take $\vk(p_{xy})=1$ and $\vk(p_{yx})=0$.

\iim{b}When $x$ and $y$ are both in the list\,
and $x$ is ranked as good as $y$,\,
we take $\vk(p_{xy})=\vk(p_{yx})=\onehalf$.

\iim{c}When $x$ is in the list and $y$ is not in it,\,
we take $\vk(p_{xy})=1$ and $\vk(p_{yx})=0$.

\iim{d}When neither $x$ nor  $y$ are in the list, 
we take $\vk(p_{xy}) =\vk(p_{yx}) =0$.

\halfsmallskip
\noindent
Instead of rule~(d), one can consider the possibility of using the following alternative:

\halfsmallskip
\iim{d$'$} When neither $x$ nor $y$ are in the list,\,
we interpret that they are considered equally good
(or equally bad),\, so we proceed as in~(b).

\halfsmallskip
\noindent
This amounts to complete each truncated ranking by appending to it all the missing options tied to each other.
\ensep
Generally speaking, however, this interpretation can be criticized in that the added information might not be really meant by the voter.

\medskip
The table that collects the numbers $v(p_{xy})$ given by (\ref{eq:cog}) for all ordered pairs $xy$ will be called the (normalized) \dfc{Llull matrix}
of the vote.
This name refers to Ramon Llull, who already considered such tables in the thirteenth century \cite{mu}.
The elements of this matrix satisfy the inequality
\begin{equation}
\label{eq:sum-less-than-one}
v(p_{xy}) + v(p_{yx}) \le 1,
\end{equation}
which is inherited from the component valuations $\vk$.
\ensep
When (\ref{eq:sum-less-than-one}) holds with the equality sign,
it means that every individual expressed a comparison (a~preference or a tie) about every pair of options.
We will refer to this situation as the \dfc{complete} case.
As we have already remarked, the revised degrees of belief need not satisfy (\ref{eq:sum-less-than-one}).


\section{Transitivity}

Let us begin by the notion of choiceness that presupposes a complete ordering:
the~right choice is the option that goes first in the right complete ordering.
So the problem focuses here on arriving at a complete ordering.
In other words, preferences are here constrained to be transitive, \ie to
satisfy the implications $p_{xy} \land p_{yz} \rightarrow p_{xz}$. On account of (\ref{eq:iff}), the latter are logically equivalent to the following clauses:
\begin{equation}
\label{eq:transitivity}
p_{xy} \lor p_{yz} \lor p_{zx},\qquad \text{for any pairwise different $x,y,z\in\ist$}
\end{equation}
(where $x,y,z$ have been relabelled). The doctrine that is made of these clauses will be referred to as the \dfc{transitivity doctrine.}


\medskip
The properties of disjoint resolvability and unquestionability of this doctrine
(obtained in Proposition~\ref{st:tech-trans} from appendix~A)
allow us to express $\urv$ directly as the result of the one-step transformation associated with the Blake canonical form, namely:
\begin{equation}
\label{eq:camins}
\urv(p_{xy}) \,=\, \Max\, \min\big(\orv(p_{x_0x_1}), \orv(p_{x_1x_2}), \dots ,\orv(p_{x_{n-1}x_n})\big),
\end{equation}
where the $\Max$ operator considers all paths $x_0x_1\dots x_n$ of length $n\ge 1$
from $x_0=x$ to $x_n=y$ with all $x_i$ pairwise different.

\medskip
In this case, our general method corresponds essentially to the method introduced in~1997 by Markus Schulze 
(posted in a mailing list about election methods; see 
\cite{sc,scbis}, \cite[p.\,228--232]{t6} and \cite{crc,cri}),
%
sometimes called the method of paths (in the incomplete case, however, it does not coincide with any of the variants given in \cite{scbis}). In the way that we have introduced it, it is clearly a~method for ranking all the candidates. Having said that, later on~(\secpar{6.4}) we will see that in the complete case its winners are quite in agreement with a doctrine that does not include transitivity but aims only at choosing the most prominent option.

As a ranking method, the method of paths complies with the following extension of the Condorcet principle introduced in 1973 by John H.~Smith 
\cite[\secpar{5}]{smith}: 
Assume that the set of candidates is partitioned in two classes~$\xst$ and~$\yst$ such that for each member of $\xst$ and every member of~$\yst$ there are more than half of the individual votes where the former is preferred to the latter; in that case,
the social ranking should also prefer each member of $\xst$ to any member of~$\yst$.
The proof can be found in
\cite[\secpar{4.7}]{scbis} 
(see also
\cite[\secpar{10}]{crc}, \cite[Thm.\,8.1]{cri}).
\ensep
Another interesting property of the method of paths is clone consistency, also known as independence of clones, which refers to the effect of replacing a single option $c$ by a set $C$ of several options similar to $c$; for more details we refer the reader to
\cite[\secpar{5.4}]{sc}, \cite[\secpar{4.6}]{scbis}
as well as 
\cite[\secpar{11}]{crc}, \cite[Thm.\,8.2 and 8.3]{cri}.

Moreover, it has also been shown \cite{crc,cri} that this method can be extended to a continuous rating method that allows to sense the closeness of two candidates at the same time that it allows to recognise certain situations that are quite opposite to a tie.

\begin{comment}
Replacing transitivity by a semiorder or interval-order character (refs): the results do not differ from the case of transitivity when only strict preferences are present or the Landau rule is used. In fact, semiorders and interval orders reduce to transitivity for complete strict preferences.
\end{comment}

\begin{comment}
To be clarified: Degrees of belief 
not only allow to model indifference
but they also allow for intransitive indifference (sorites paradox) \cite{ref}
within a logical framework where all-or-none preferences are constrained to be transitive
(and fractional ones too in the sense that $v_{xz} \ge \min(v_{xy},v_{yz})$)
Example: $\ist=\{a,b,c\}$, $v_{xy} = \onehalf$ for any $xy\neq ac$, $v_{ac}=1$\,(???).
Transitivity is equivalent to certain generalizations of it (as an effect of (\ref{eq:iff})?)
\end{comment}

\medskip 

\section{Supremacy}

Instead of constraining the binary preferences between several options to 
form a total order, one can require only the existence of a supreme option,
\ie an option that is preferred to any other.
Such a constraint is specified by the \textit{disjunctive} normal form
\begin{equation}
\label{eq:one-best-original}
\bigvee_{x\in\ist}\,\,\bigwedge_{\substack{y\in\ist\\y\neq x}}\, p_{xy}.
\end{equation}
Instead of directly bringing (\ref{eq:one-best-original}) into conjunctive normal form,
one can arrive at a much shorter conjunctive formulation 
by considering the propositions $\sprm_x:$ `$x$ is preferred to any other member of $\ist$'
($x\in\ist$), which are related to the $p_{xy}$ by the double implications
\begin{equation}
\label{eq:def-best}
\sprm_x \,\,\leftrightarrow\,\, \bigwedge_{\substack{y\in\ist\\y\neq x}} p_{xy}.
\end{equation}
By proceeding in this way, we are led to consider the set of propositions $\piset=\{\,\sprm_x\mid x\in\ist\,\} \cup \{\,\nt\sprm_x\mid x\in\ist\,\} \cup \{\,p_{xy}\mid x,y\in\ist,\,x\neq y\,\}$ together with the doctrine formed by the following clauses:
\begin{alignat}{2}
\label{eq:beats-all-implies-best}
&\sprm_x\, \lor\, \bigvee_{\substack{y\in\ist\\y\neq x}}\, p_{yx},\qquad &&\text{for any $x\in\ist$;}
\\
\label{eq:best-implies-beats-all}
&\nt \sprm_x\, \lor\,\, p_{xy},\qquad &&\text{for any two different $x,y\in\ist$;}
\\[3.5pt]
\label{eq:best-exists}
&\bigvee_{x\in\ist}\, \sprm_x.
\end{alignat}
We will refer to it as the \dfd{supremacy doctrine}.

\medskip
Again, one can see that the preceding clauses lead to the same upper revised valuation as the corresponding full Blake canonical form (Prop.~\ref{st:tech-sprm1}). Here we will only notice that the Blake canonical form includes the clauses
\begin{equation}
\label{eq:best-is-unique}
\nt\sprm_x\, \lor\,\, \nt\sprm_y,\qquad \text{for any two different $x,y\in\ist$,}
\end{equation}
which are obtained by disjoint resolution between (\ref{eq:best-implies-beats-all}) and the clause of the same form with $x$ and $y$ interchanged with each other (recall that we identify $\nt p_{xy}$ with $p_{yx}$).


\medskip
The clauses (\ref{eq:best-is-unique}) assert that one cannot have two supreme options.
Applying the definite consistency theorem (Theorem~\ref{st:dec}) to these clauses ensures the following fact:
When $\sprm_x$ is accepted, then $x$ is the only option with this property.
On~the other hand, the same theorem applied to (\ref{eq:best-exists})
ensures the following fact:
When $\sprm_z$ is rejected for any $z\neq x$, then $\sprm_x$ is accepted.
By~taking into account that a proposition need not be accepted or rejected but it can be left undecided,
the preceding statement is equivalent to the following one:
When $\sprm_z$ is not accepted for any $z\in\ist$, 
then $\sprm_z$ is undecided for more than one $z\in\ist$.
In~the sequel, an option for which $\sprm_x$ is not rejected will be called a \dfc{supremacy winner.}

\medskip
The one-step revision transformation $v\mapsto\utv$ associated with (\ref{eq:beats-all-implies-best}--\ref{eq:best-exists})
reads as follows: For any $x,y\in\ist$:
\begin{alignat}{2}
\label{eq:one-step-best}
&\utv(\sprm_x) \,&&=\, \max\Big(\, v(\sprm_x),\,
\min_{y\neq x} v(\nt \sprm_y),\,
\min_{y\neq x} v(p_{xy}) \,\Big),
\\
\label{eq:one-step-not-best}
&\utv(\nt \sprm_x) \,&&=\, \max\Big(\, v(\nt \sprm_x),\,
\max_{y\neq x} v(p_{yx}) \,\Big),
\\
\label{eq:one-step-pxy}
&\utv(p_{xy}) \,&&=\, \max\Big(\, v(p_{xy}),\,
v(\sprm_x),\, \min\big(v(\nt \sprm_y),\,\min_{\substack{z\neq x\\z\neq y}} v(p_{yz})\big)
\,\Big).
\end{alignat}

We will use the following notations: 
\begin{equation}
\label{eq:defs-minrow-maxcol-minrw}
\minrow_x = \min_{y\neq x}\,\orv(p_{xy}),\quad
\maxcol_x = \max_{y\neq x}\,\orv(p_{yx}),\quad
\minrw_{xy} = \min_{\substack{z\neq x\\z\neq y}}\,\orv(p_{xz}).
\end{equation}
These definitions immediately imply that 
\begin{alignat}{3}
\label{eq:ineq1}
&\minrow_x \,&&\le\, \orv(p_{xy}) \,&&\le\, \maxcol_y,\qquad \text{whenever $x\neq y$.}
\\[2.5pt]
\label{eq:ineq2}
&\minrw_{xy} \,&&\le\, \orv(p_{xz}) \,&&\le\, \maxcol_z,\qquad \text{whenever $z\notin\{x,y\}$.}
\end{alignat}

In the sequel we will have to look at the successive valuations~$\ntv{n}$
that are obtained by iterating the transformation (\ref{eq:one-step-best}--\ref{eq:one-step-pxy})
starting from $\ntv{0}=\orv$.
Later on, it will be convenient to allow $n$ to take negative values by putting
$\ntv{n}(p_{xy}) = 0$ for $n < 0$.
We will also make use of the quantities analogous to those of (\ref{eq:defs-minrow-maxcol-minrw})
with $\ntv{n}$ substituted for $\orv$.
These quantities will be denoted by $\nmaxcol{n}_x,\nminrow{n}_x,\nminrw{n}_{xy}$.
Obviously, they satisfy inequalities analogous to (\ref{eq:ineq1}--\ref{eq:ineq2}).

\paragraph{5.1}{\textbf{The minimax rule}}

\medskip
Assume that our choice must be based solely on the Llull matrix $(v(p_{xy}))$.
In principle, this matrix does not give (direct) information about the (collective) degrees of belief
for $\sprm_x$ and $\nt\sprm_x$. So, it makes sense to take
\begin{equation}
\label{eq:zero-supreme}
v(\sprm_x)\,=\,v(\nt \sprm_x)\,=\,0,\quad \text{for any $x\in\ist$}.
\end{equation}

\medskip
\begin{proposition}\hskip.5em
\label{st:laia}
For the initial values \textup{(\ref{eq:zero-supreme})}, the supremacy doctrine
gives
\begin{alignat}{3}
&\urv(\sprm_x) \,&&=\, \utv'(\sprm_x) \,&&=\, \min_{z\neq x}\, \maxcol_z,
\\
&\urv(\nt \sprm_x) \,&&=\, \utv(\nt \sprm_x) \,&&=\, \maxcol_x,
\\[3.5pt]
&\urv(p_{xy}) \,&&=\, \utv''(p_{xy}) \,&&=\, \max\,(v(p_{xy}),\,\min_{z\neq x} \maxcol_z).
\end{alignat}
\end{proposition}

\begin{proof}\hskip.5em
Let us introduce the initial values (\ref{eq:zero-supreme}) in (\ref{eq:one-step-best}--\ref{eq:one-step-pxy}).
Starting from $\ntv{n}(\nt \sprm_x)$ and using Lemma~\ref{st:lem-best-option},
we successively obtain:
\begin{alignat}{2}
\label{eq:n00-step-not-best}
&\ntv{n}(\nt \sprm_x) \,&&=\, \nmaxcol{n-1}_x,
\\[2.5pt]
\label{eq:n00-step-best}
&\ntv{n}(\sprm_x) \,&&=\, \max\big(\, \min_{y\neq x} \nmaxcol{n-2}_y,\, \nminrow{n-1}_x \,\big),
\\
\label{eq:n00-step-pxy}
&\ntv{n}(p_{xy}) \,&&=\, \max\Big(\, \orv(p_{xy}),\,
\min_{z\neq x} \nmaxcol{n-3}_z,\, \nminrow{n-2}_x,\, \min\big(\nmaxcol{n-2}_y,\,\nminrw{n-1}_{yx})\big)
\,\Big).
\end{alignat}
Let us now plug (\ref{eq:n00-step-pxy}) into the definition of $\nmaxcol{n}_x$. 
By making use of the inequalities (\ref{eq:ineq1}--\ref{eq:ineq2}) and their $n$-th counterparts,
one easily arrives at the in\-equality $\nmaxcol{n}_x \le \max(\maxcol_x, \nmaxcol{n-3}_x, \nmaxcol{n-2}_x)$.
By induction, it follows that $\nmaxcol{n}_x \le \maxcol_x$ for $n\ge0$.
Since we also know that $\nmaxcol{n}_x$ is not decreasing, 
we get
\begin{equation}
\label{eq:maxcol-constant}
\nmaxcol{n}_x \,=\, \maxcol_x,\qquad\text{for $n\ge0$.}
\end{equation}
Finally, by plugging this result into (\ref{eq:n00-step-not-best}--\ref{eq:n00-step-pxy})
and making use of the inequalities of the type (\ref{eq:ineq1}--\ref{eq:ineq2}), one arrives at the conclusion that
\begin{alignat}{3}
\label{eq:nn-step-not-best}
&\ntv{n}(\nt \sprm_x) \,&&=\, \maxcol_x,\qquad&&\text{for $n\ge1$;}
\\[2.5pt]
\label{eq:nn-step-best}
&\ntv{n}(\sprm_x) \,&&=\, \min_{z\neq x}\, \maxcol_z,\qquad&&\text{for $n\ge2$;}
\\
\label{eq:nn-step-pxy}
&\ntv{n}(p_{xy}) \,&&=\, \max\,(v(p_{xy}),\,\min_{z\neq x} \maxcol_z),\qquad&&\text{for $n\ge3$.\qedhere}
\end{alignat}
\end{proof}

\medskip
\begin{corollary}\hskip.5em
For the initial values \textup{(\ref{eq:zero-supreme})}, the supremacy winners are the options $x\in\ist$ that minimize $\maxcol_x = \max_{y\neq x} v(p_{yx})$.
\end{corollary}

\noindent
We refer to this rule as the \dfc{minimax} rule. In the voting literature, this term is sometimes associated with several different rules, in which case the preceding rule is specifically known as ``pairwise opposition''. In the complete case ($v(p_{xy})+v(p_{yx})=1$) it coincides with the \dfc{maximin} rule, \ie choosing the option $x\in\ist$ that maximizes $\minrow_x = \min_{y\neq x}\,v(p_{xy})$ \cite[p.\,212--213]{t6}, which complies with Condorcet's majority principle (see \secpar{6.1}).
In the general incomplete case, however, the minimax rule does not comply with Condorcet's principle.



\paragraph{5.2}{\textbf{The plurality rule.}}

\medskip
When the Llull matrix comes from preferential voting 
in the sense that every vote is an ordered list
(possibly restricted to a subset of most preferred options),
then the preceding treatment admits of a serious objection.
\ensep
In fact, in that case
it is natural to adopt
certain specific values as initial degrees of (collective) belief in $\sprm_x$ and $\nt\sprm_x$,
namely and specifically, the fraction $\plu_x$ of votes where $x$ is placed at the top of the list,
and that of those where some other option is placed at the top:
\begin{equation}
\label{eq:plu-best}
v(\sprm_x)\,=\, \plu_x,\quad
v(\nt \sprm_x)\,=\, \aplu_x,\quad
\text{for any $x\in\ist$,}
\end{equation}
where $\aplu_x = \sum_{y\neq x}\plu_y$. 
Since $\sprm_x$ cannot be true for two different options \hbox{---clause~(\ref{eq:best-is-unique})---}
in the event of a vote that ties $k$ options at the top,
it makes sense to count it as $1/k$-th of a vote for each of the top-placed options.
In the sequel we will refer to~$\plu_x$ as the \dfc{plurality fraction} of~$x$,
and $\aplu_x$ will be called the \dfc{antiplurality fraction} of~$x$.
\ensep
The values of $\plu_x$ cannot be read from the Llull matrix except in very few special cases.
However, they are easily obtained from the votes themselves.
Using this additional information should lead to better grounded results.
\ensep
Yet we get something rather unexpected:

\medskip
\begin{proposition}\hskip.5em
\label{st:plu-thm}
For the initial values \textup{(\ref{eq:plu-best})}, the supremacy doctrine
gives
\begin{alignat}{3}
\label{eq:plu1}
&\urv(\sprm_x) \,&&=\, \utv(\sprm_x) \,&&=\, \min_{z\neq x}\aplu_z,
\\
\label{eq:plu2}
&\urv(\nt \sprm_x) \,&&=\, \orv(\nt\sprm_x) \,&&=\, \aplu_x,
\\[3.5pt]
\label{eq:plu3}
&\urv(p_{xy}) \,&&=\, \utv'(p_{xy}) \,&&=\, \max\,(v(p_{xy}),\,\min_{z\neq x}\aplu_z).
\end{alignat}
\end{proposition}

\begin{proof}\hskip.5em
Let us begin by noticing that the plurality and antiplurality fractions $\plu_x$ and $\aplu_x$ are related to the entries of the Llull matrix 
in the following way:
\begin{equation}
\label{eq:fsandwich}
\plu_x \,\le\, v(p_{xy}) \,\le\, \aplu_y.
\end{equation}
This is an immediate consequence of the definitions when the votes are strict rankings:
If $x$ is placed at the top, then $x$ is preferred to any other option~$y$;
on the other hand, if $x$ is preferred to $y$, then the top option cannot be~$y$.
A~vote that ties $k\ge2$~options at the top contributes also to the three terms of~(\ref{eq:fsandwich})
in agreement with the stated inequalities; in particular, if both $x$ and $y$ are placed at the top,
the respective contributions are $1/k \le 1/2 \le (k-1)/k$.
\ensep
From (\ref{eq:fsandwich}) and (\ref{eq:defs-minrow-maxcol-minrw}--\ref{eq:ineq2}) it follows that 
\begin{alignat}{4}
\label{eq:fsandwich1}
&\plu_x \,&&\le\, \minrow_x \,&&\le\, \maxcol_y \,&&\le\, \aplu_y,\qquad \text{whenever $x\neq y$.}
\\[2.5pt]
\label{eq:fsandwich2}
&\plu_x \,&&\le\, \minrw_{xy} \,&&\le\, \maxcol_z \,&&\le\, \aplu_z,\qquad \text{whenever $z\notin\{x,y\}$.}
\end{alignat}

Let us introduce the initial values (\ref{eq:plu-best}) in (\ref{eq:one-step-best}--\ref{eq:one-step-pxy}). Using Lemma~\ref{st:lem-best-option} we get (for any $n\ge0$):
\begin{alignat}{2}
\label{eq:np-step-best}
&\ntv{n}(\sprm_x) \,&&=\, \max\big(\, \plu_x,\,
\min_{y\neq x} \ntv{n-1}(\nt \sprm_y),\, \nminrow{n-1}_x \,\big),
\\
\label{eq:np-step-not-best}
&\ntv{n}(\nt \sprm_x) \,&&=\, \max\big(\, \aplu_x,\, \nmaxcol{n-1}_x \,\big),
\\[2.5pt]
\label{eq:np-step-pxy}
&\ntv{n}(p_{xy}) \,&&=\, \max\Big(\, \orv(p_{xy}),\,
\ntv{n-1}(\sprm_x),\, \min\big(\ntv{n-1}(\nt \sprm_y),\,\nminrw{n-1}_{yx}\big)
\,\Big).
%
\end{alignat}

We will prove by induction that the following inequality holds for any $n\ge1$:
\begin{equation}
 \label{eq:ind4}
 \nmaxcol{n}_x \,\le\, \aplu_x,\qquad\text{for every $x$.}
\end{equation}
For $n=1$ this follows from (\ref{eq:np-step-pxy}) because of the initial values (\ref{eq:plu-best}) and the inequalities (\ref{eq:fsandwich}) and (\ref{eq:fsandwich2}). In fact, we get
$$\utv(p_{xy})=\max\Big(\, \orv(p_{xy}),\,
\plu_x,\, \min\big(\aplu_y,\,\minrw_{yx}\big)
\,\Big)\le \, \aplu_y, $$
which implies $\maxcol'_y\,=\,\max_{x\ne y} \utv(p_{xy}) \,\le\, \aplu_y.$

Assume now that (\ref{eq:ind4}) holds for a given $n\ge1$.
Using (\ref{eq:np-step-best}--\ref{eq:np-step-pxy}), we arrive successively at the following facts:
\begin{equation}
\label{eq:ind1}
\ntv{k}(\nt\sprm_x) 
\,=\, \max\big(\, \aplu_x,\, \nmaxcol{k-1}_x \,\big)
\,=\, \aplu_x,\quad \text{whenever $1\le k\le n+1$,}
\end{equation}
since $\nmaxcol{k-1}_x\le\,\nmaxcol{n}_x\le\,\aplu_x$;
\begin{align}
\label{eq:ind2}
\ntv{n}(\sprm_x)
&\,=\, \max\big(\,\plu_x,\,\min_{y\neq x}\,\ntv{n-1}(\nt \sprm_y),\,\nminrow{n-1}_x\,\big)
\nonumber\\
&=\, \max\big(\,\plu_x,\,\min_{y\neq x}\,\aplu_y,\,\nminrow{n-1}_x\,\big)
\,=\, \min_{y\neq x}\,\aplu_y,
\end{align}
since $\plu_x\,\le\,\minrow_x\,\le\,\nminrow{n-1}_x\le\,\nmaxcol{n-1}_y\le\,\nmaxcol{n}_y\le\,\aplu_y$ for any $y\ne x$;
\begin{align}
\label{eq:ind3}
\ntv{n+1}(&p_{xy})
\,=\, \max\Big(\, \orv(p_{xy}),\,
\ntv{n}(\sprm_x),\, \min\big(\ntv{n}(\nt \sprm_y),\,\nminrw{n}_{yx}\big)
\,\Big)\nonumber\\
&=\, \max\Big(\, \orv(p_{xy}),\,\min_{z\neq x} \aplu_z
 ,\, \min\big( \aplu_y,\,\nminrw{n}_{yx}\big)
\,\Big)\,=\,\max\Big(\, \orv(p_{xy}),\,\min_{z\neq x} \aplu_z
\,\Big),
\end{align}
since $\nminrw{n}_{yx}\le\,\nmaxcol{n}_z\le\,\aplu_z$ for any $z\ne x,y$; \,and finally
\begin{equation*}
\nmaxcol{n+1}_y
\,=\, \max_{x\ne y}\,\ntv{n+1}(p_{xy})
\,=\, \max_{x\ne y}\Big(\, \max\big(\, \orv(p_{xy}),\,\min_{z\neq x} \aplu_z \,\big) \,\Big)
\,\le\, \aplu_y,
\end{equation*}
because of (\ref{eq:fsandwich}). This finishes the proof of (\ref{eq:ind4}). As a byproduct we have obtained also (\ref{eq:ind1}--\ref{eq:ind3}), that entail the equalities (\ref{eq:plu1}--\ref{eq:plu3}) claimed in the proposition.
\end{proof}

\medskip
\begin{corollary}\hskip.5em
\label{st:plu-cor}
For the initial values \textup{(\ref{eq:plu-best})},
the supremacy winners are the plurality winners, 
\ie the options $x\in\ist$ that maximize the plurality fraction~$\plu_x$.
\end{corollary}

\begin{proof}
Recall that we have defined a supremacy winner as an option $x\in\ist$ for with $\sprm_x$ is not rejected, \ie such that $\urv(\sprm_x)\ge\urv(\nt\sprm_x)$. In view of (\ref{eq:plu1}--\ref{eq:plu2}), this is equivalent to say that $x$ minimizes $\aplu_x$. Finally, since $\aplu_x=\sum_{z\neq x}\plu_z= (\sum_{z\in\ist}\plu_z) - \plu_x$,
minimizing $\aplu_x$ is equivalent to maximizing the plurality fraction $\plu_x$.
\end{proof}

This is quite embarrassing:
We started from a method that in the complete case complies with the Condorcet principle,
thus ruling out the quite objectionable plurality rule,
and now, by adding more information, we have fallen back into the plurality rule\,!
\ensep
A little reflection shows that
in order to avoid the main drawback of the plurality rule
one should not look for supremacy,
but for something slightly different.
In fact, the main objection against the plurality rule 
---raised by Borda in his seminal paper of 1770--84 \cite[ch.\,5]{mu}---
is that one can have a majority of voters for which the plurality winner is
the \emph{worst} option,
which is certainly quite undesirable.
Notice that what matters here is the opposition between `best' and `worst', whereas the supremacy doctrine has to do  
with the opposition between `best' and `not~best'.

\section{\textbf{Prominence.}}

In order to properly deal with the `best-worst' opposition, one is led to 
replace suprem\-acy by a weaker concept whose connection to preferences requires only that
`best' implies the presence of that concept and `worst' (instead of `not~best') implies the lack of it.
This concept could be viewed as a sort of tempered supremacy.
We will refer to it as `prominence'.
More properly speaking,
and using the notation $\temp_x$ to represent the proposition `$x$~is prominent',
the connection between this concept and preferences is given by the following two implications:
\ensep
if~$x$ is preferred to any other option, then $x$~is prominent:
$\bigwedge_{y\neq x}\, p_{xy} \rightarrow \temp_x$;\ensep
if every option other than~$x$ is preferred to~$x$, then $x$~is not prominent: 
$\bigwedge_{y\neq x}\, p_{yx} \rightarrow \nt\temp_x$.
\ensep
In~conjunctive normal form these implications read as follows:
\begin{alignat}{2}
\label{eq:best-implies-tx}
&\temp_x\, \lor\, \bigvee_{y\neq x}\, p_{yx},\qquad &&\text{for any $x\in\ist$;}
\\
\label{eq:worst-implies-ntx}
&\nt\temp_x\, \lor\, \bigvee_{y\neq x}\, p_{xy},\qquad &&\text{for any $x\in\ist$;}
\end{alignat}

\medskip
Concerning the initial values for $v(\temp_x)$ and $v(\nt \temp_x)$,
we will take simply 
\begin{equation}
\label{eq:zero-suitable}
v(\temp_x)\,=\,v(\nt \temp_x)\,=\,0,\qquad \text{for any $x\in\ist$}.
\end{equation}
One could argue that in the case of preferential voting one should proceed in a different way:
In accordance with the implication $\bigwedge_{y\neq x}\, p_{xy} \rightarrow \temp_x$ 
---contained in (\ref{eq:best-implies-tx})---
every top placing of~$x$ is a piece of evidence in favour of $\temp_x$.
Similarly, every last placing of~$x$ is a piece of evidence in favour of~$\nt\temp_x$
---by the implication contained in (\ref{eq:worst-implies-ntx}).
So, one should take
\begin{equation}
\label{eq:zero-suitablebis}
v(\temp_x)\,=\,\plu_x,\quad v(\nt \temp_x)\,=\,\ell_x,\qquad \text{for any $x\in\ist$},
\end{equation}
where $\ell_x$ denotes the fraction of votes where $x$ is placed last
(a vote that ties $k$ options at the bottom being counted as $1/k$-th of a vote for each of the bottom-placed options).
\ensep
In the supremacy doctrine a similar change in the initial values led to an entirely different result.
Here, however, the initial values (\ref{eq:zero-suitablebis}) lead to the same result as (\ref{eq:zero-suitable}).
This happens because instead of the second inequality of (\ref{eq:fsandwich}) here we have the following one:
$\ell_y \le \orv(p_{xy})$. This inequality, together with the first inequality of (\ref{eq:fsandwich}),
has the following consequence: no matter whether we start from (\ref{eq:zero-suitable}) or from (\ref{eq:zero-suitablebis}),
we get $\orv'(\temp_x)\ge \min_{y\neq x} v(p_{xy})\ge\plu_x$ as well as $\orv'(\nt\temp_x)\ge\min_{y\neq x} v(p_{yx})\ge\ell_x$.
In~fact, the initial values (\ref{eq:zero-suitablebis}) are based on 
the implications $\bigwedge_{y\neq x}\, p_{xy} \rightarrow \temp_x$
and $\bigwedge_{y\neq x}\, p_{yx} \rightarrow \nt\temp_x$,
so they are doing part of the job that will be done anyway by the revision transformation.
\ensep
If the doctrine includes (\ref{eq:best-implies-tx}) but not (\ref{eq:worst-implies-ntx})
---as it will be the case in~\secpar{6.2}---
then the preceding considerations hold only with respect to the first equality of (\ref{eq:zero-suitablebis}).

\paragraph{6.1}{\textbf{The Condorcet principle.}}

\medskip
The implication $\bigwedge_{y\neq x}\, p_{xy} \rightarrow \temp_x$
that is coded in clause (\ref{eq:best-implies-tx}) is akin to 
the celebrated \dfc{Concorcet principle.} This principle has the two following versions: 

\smallskip
\newcommand\llmpw{\textup{M1}}
\condition{\llmpw}{Condorcet principle (majority version)} If an option $x$ has the property that $\orv(p_{xy})>\onehalf$ for any $y\neq x$, then $x$ must be chosen as the winner.

\smallskip
\newcommand\llcpw{\textup{M1$'$}}
\condition{\textup{M1\rlap{$'$}}}{Condorcet principle (margin version)}
If an option $x$ has the property that $\orv(p_{xy})>\orv(p_{yx})$ for any $y\neq x$, then $x$ must be chosen as the winner.

\smallskip
\noindent
In the complete case $\orv(p_{xy})+\orv(p_{xy})=1$ (where the Condorcet principle was originally proposed) these two conditions are equivalent to each other.
\ensep
Generally speaking, however, condition~\llmpw\ is weaker than~\llcpw\ 
(which makes the former more compatible with other desirable properties,
as it was remarked in \cite[\secpar{1.4}]{cri}).

On the other hand, the implication $\bigwedge_{y\neq x}\, p_{yx} \rightarrow \nt\temp_x$
coded in clause (\ref{eq:worst-implies-ntx}) corresponds to
the dual statement that is usually referred to as the ``Condorcet loser criterion'',
also with two versions: the majority one requiring $\orv(p_{yx})>\onehalf$ for any $y\neq x$,
and the margin one requiring $\orv(p_{yx})>\orv(p_{xy})$ for any $y\neq x$.
In both versions, the conclusion is that $x$ must then be deemed a loser.

In the sequel, the term \dfc{Condorcet winner} [\,resp.~\dfc{loser}] will be understood in the majority sense, \ie to denote an option $x$ with the property that $\orv(p_{xy})>\onehalf$ \,[\,resp.~$\orv(p_{yx})>\onehalf$]\, for any $y\neq x$.

\medskip
A major difference between our point of view and that of the Condorcet principle 
is that the latter, in both versions \llmpw\ and \llcpw,
looks at whether a certain particular situation happens in the initial (collective) degrees of belief~$\orv$.
If it does not happen, then no conclusion is arrived at. 
\ensep
In contrast, our method will take the implication $\bigwedge_{y\neq x}\, p_{xy} \rightarrow \temp_x$,
\ie clause (\ref{eq:best-implies-tx}),
as a guide for revising those initial degrees of belief
so as to arrive at a conclusion consistent with that implication.
\ensep
In accordance with the definite consistency theorem (Theorem~\ref{st:dec}),  
we will have the following property akin to \llcpw: 
If an option $x$ satisfies $\urv(p_{xy})>\urv(p_{yx})$ for any $y\neq x$,
then it satisfies also $\urv(\temp_x)>\urv(\nt\temp_x)$,
\ie $x$ is accepted as a prominent option.

This property will be satisfied whenever the doctrine contains the clause (\ref{eq:best-implies-tx}).
However, these need not be the only options accepted as prominent ones.
Depending on which other clauses are present in the doctrine, other options might get accepted too.

Another major difference between our point of view and that of the Condorcet principle 
is that the latter, also in both versions \llmpw\ and \llcpw, 
aims at finding out the winner, \ie choosing a single option,
whereas here we aim, in principle, at finding out all prominent options.
\ensep
In fact, in contrast to the supremacy doctrine,
the clauses (\ref{eq:best-implies-tx}--\ref{eq:worst-implies-ntx})
allow for the possibility of having several prominent options or having none of them.
\ensep
To~the effect of making a single choice,
we will consider two different approaches.
\ensep
The first one, followed in \secpar{6.2--6.3},
is simply to select
the option(s) $x$ for which the proposition $\temp_x$ gets a highest acceptability,
which corresponds to deciding by a large margin. 
\ensep
The second approach,
is to impose existence and uniqueness
as part of the doctrine.

As we will see in \secpar{6.4}, imposing existence and uniqueness motivates a~more comprehensive prominence doctrine that will satisfy not only the majority version of the Condorcet principle,
but also certain generalizations of it.

\paragraph{6.2}{\textbf{The maximin rule.}}

\medskip
In this section we show that the well-known maximin rule \cite[p.\,212--213]{t6},
\ie selecting the $x\in\ist$ that maximizes $\minrow_x=\min_{y\neq x} v(p_{xy})$,
corresponds exactly to keeping only the clauses (\ref{eq:best-implies-tx})
and applying the highest acceptability approach.

Often attributed to Simpson~(1969) and Kramer~(1977),
in actual fact the maximin rule appears already in Duncan Black's celebrated work of 1958
\cite[(i) in p.\,208]{bl}. 
The term `maximin' that we are using is taken from \cite{t6}.
Having said that, all of these authors limited their attention to the complete case $v(p_{xy})+v(p_{yx})=1$, 
where maximizing $\minrow_x=\min_{y\neq x} v(p_{xy})$ is equivalent to minimizing $\maxcol_x=\max_{y\neq x} v(p_{yx})$.
In~the general case, however, one must distinguish between these two rules,
that we call respectively `maximin' and `minimax' (see \secpar{5.1}).


\medskip
The one-step revision transformation $v\mapsto\utv$ associated with the clauses (\ref{eq:best-implies-tx})
takes the following form: For any $x,y\in\ist$:
\begin{alignat}{2}
\label{eq:one-step-tx-maximin}
&\utv(\temp_x) \,&&=\, \max\Big(\,
v(\temp_x),\, \min_{y\neq x} v(p_{xy}) \,\Big),
\\
\label{eq:one-step-ntx-maximin}
&\utv(\nt \temp_x) \,&&=\, v(\nt \temp_x),
\\
\label{eq:one-step-pxy-maximin}
&\utv(p_{xy}) \,&&=\, \max\Big(\, v(p_{xy}),\,
\min\big(v(\temp_x),\,\min_{\substack{z\neq x\\z\neq y}} v(p_{zx})\big)
\,\Big),
\end{alignat}

\begin{proposition}\hskip.5em
\label{st:maximin}
For initial valuations satisfying \textup{(\ref{eq:zero-suitable})},
the doctrine \textup{(\ref{eq:best-implies-tx})} is unquestionable for all of its propositions and it gives
\begin{alignat}{3}
\label{eq:maximin-eq1}
&\urv(\temp_x) \,&&=\, \utv(\temp_x) \,&&=\, \minrow_x,
\\
\label{eq:maximin-eq2}
&\urv(\nt \temp_x) \,&&=\, \orv(\nt \temp_x) \,&&=\, 0,
\\
\label{eq:maximin-eq3}
&\urv(p_{xy}) \,&&=\, \orv(p_{xy}),
\end{alignat}
where $\minrow_x = \min_{y\neq x}\,\orv(p_{xy})$.
Therefore, the most prominent option is the
$x\in\ist$ that maximizes $\minrow_x$.
\end{proposition}

\begin{proof}\hskip.5em
Equalities (\ref{eq:maximin-eq1}--\ref{eq:maximin-eq3}) are easily obtained by making use of Lemma~\ref{st:lem-best-option} and the inequality $\orv(p_{xy}) \ge \minrow_x$. The unquestionability statement is simply a consequence of having obtained $\urv=\utv$.
\end{proof}


\paragraph{6.3}{\textbf{Symmetric prominence.}}

\medskip
In this section, we consider the doctrine that includes both (\ref{eq:best-implies-tx}) and (\ref{eq:worst-implies-ntx}) and we choose the option that maximizes the (revised) acceptability of $\temp_x$.
We refer to it as the symmetric prominence method.
As a consequence of including also the clauses (\ref{eq:worst-implies-ntx}), 
the winner need not be the same as the maximin one.
What is more, we will see that a Condorcet winner need not be the symmetric prominence winner.
However, in this doctrine, a Condorcet winner is always accepted as a prominent option.
Besides, the fact that this doctrine is symmetric under negation
ensures also that a Condorcet loser is always rejected as a prominent option.



\medskip
The one-step revision transformation $v\mapsto\utv$ associated with (\ref{eq:best-implies-tx}--\ref{eq:worst-implies-ntx})
takes the following form: For any $x,y\in\ist$:
\begin{alignat}{2}
\label{eq:one-step-tx-nct-prominence}
&\utv(\temp_x) \,&&=\, \max\Big(\,
v(\temp_x),\, \min_{y\neq x} v(p_{xy}) \,\Big),
\\
\label{eq:one-step-ntx-nct-prominence}
&\utv(\nt \temp_x) \,&&=\, \max\Big(\,
v(\nt \temp_x),\, \min_{y\neq x} v(p_{yx}) \,\Big),
\\
\label{eq:one-step-pxy-nct-prominence}
&\utv(p_{xy}) \,&&=\, \max\Big(\, v(p_{xy}),\,
\min\big(v(\temp_x),\,\min_{\substack{z\neq x\\z\neq y}} v(p_{zx})\big),\,
\min\big(v(\nt \temp_y),\,\min_{\substack{z\neq x\\z\neq y}} v(p_{yz})\big)
\,\Big),
\end{alignat}

\begin{proposition}\hskip.5em
\label{st:laia2}
For initial valuations satisfying \textup{(\ref{eq:zero-suitable})},
the symmetric prominence doctrine is unquestionable for all of its propositions and it gives
\begin{alignat}{3}
\label{eq:simpro-eq1}
&\urv(\temp_x) \,&&=\, \utv(\temp_x) \,&&=\, \minrow_x,
\\
\label{eq:simpro-eq2}
&\urv(\nt \temp_x) \,&&=\, \utv(\nt \temp_x) \,&&=\, \mincol_x,
\\
\label{eq:simpro-eq3}
&\urv(p_{xy}) \,&&=\, \orv(p_{xy}),
\end{alignat}
where $\minrow_x = \min_{y\neq x}\,\orv(p_{xy})$ and $\mincol_x = \min_{y\neq x}\,\orv(p_{yx})$.
\end{proposition}

\begin{proof}\hskip.5em
Equalities (\ref{eq:simpro-eq1}--\ref{eq:simpro-eq3}) are easily obtained by making use of Lemma~\ref{st:lem-best-option} and the inequalities $\orv(p_{xy}) \ge \minrow_x$, $\orv(p_{xy}) \ge \mincol_y$. The unquestionability statement is simply a consequence of having obtained $\urv=\utv$.
\end{proof}

\medskip
\begin{corollary}\hskip.5em
\label{st:condorcet-nc}
Whenever there is a Condorcet winner,
the symmetric prominence method 
accepts it as a prominent option.
\end{corollary}

\begin{proof}\hskip.5em
It suffices to notice that $x$ being a Condorcet winner implies $\minrow_x > \textstyle{\onehalf} > \mincol_x$.
\end{proof}

\medskip
The symmetric prominence doctrine, formed by clauses (\ref{eq:best-implies-tx}) and (\ref{eq:worst-implies-ntx}) is symmetric under negation,
\ie the substitution that interchanges $\temp_x$~and $\nt\temp_x$ as well as $p_{xy}$ and $p_{yx}$.
As a consequence, the preceding proposition is accompanied here by the following one:

\begin{proposition}\hskip.5em
\label{st:condorcet-loser-nc}
Whenever there is a Condorcet loser,
\ie an option $x$ such that $\orv(p_{yx})>\onehalf$ for any $y\neq x$,
the symmetric prominence method 
rejects it as a prominent option.
\end{proposition}

\remark In spite of Corollary~\ref{st:condorcet-nc}, the Condorcet winner can differ from the symmetric prominence winner, \ie the option whose prominence gets a highest acceptability. A~simple example is the following:
3~$a\cdsucc b\cdsucc c$,
2~$b\cdsucc c\cdsucc a$,
with the following Llull matrix:
\begin{equation}
\begin{tabular}{|c|c|c|}
\hlinestrut
\diaglabel{a}& 3 & 3\\
\hlinestrut
2 & \diaglabel{b} & 5\\
\hlinestrut
2 & 0 & \diaglabel{c}\\
\hline
\end{tabular}
\,,
\end{equation}
where one easily checks that the Condorcet winner $a$ gets $\Urv(\temp_a) - \Urv(\nt\temp_a) = 3-2 = 1$, but $\Urv(\temp_b) - \Urv(\nt\temp_b) = 2-0 = 2$.

\paragraph{6.4}{\textbf{Comprehensive prominence}}

\paragraph{6.4.1} To the effect of making a single choice, 
one can go for supplementing
the~symmetric prominence doctrine (\ref{eq:best-implies-tx}--\ref{eq:worst-implies-ntx})
with two additional clauses postulating the existence and uniqueness of a prominent option,
namely: $\bigvee_{z\in\ist}\, \temp_z$ (existence),
and $\temp_x \rightarrow \nt\temp_y$ (uniqueness).
\ensep
Let us write down all of these clauses together:
\begin{alignat}{2}
\label{eq:best-implies-tx-bis}
\hbox to0pt{\hss\small$(\ref{eq:best-implies-tx})$\hskip18mm} 
&\temp_x\, \lor\, \bigvee_{y\neq x}\, p_{yx},\qquad &&\text{for any $x\in\ist$;}
\\
\label{eq:worst-implies-ntx-bis}
\hbox to0pt{\hss\small$(\ref{eq:worst-implies-ntx})$\hskip18mm} 
&\nt\temp_x\, \lor\, \bigvee_{y\neq x}\, p_{xy},\qquad &&\text{for any $x\in\ist$;}
\\
\label{eq:tprominent-exists}
&\bigvee_{z\in\ist}\, \temp_z\,;
\\[3.5pt]
\label{eq:tprominent-is-unique}
&\nt\temp_x\, \lor\, \nt\temp_y,\qquad &&\text{for any two different $x,y\in\ist$.}
\end{alignat}

\medskip
Let us see which clauses derive from 
(\ref{eq:best-implies-tx-bis}--\ref{eq:tprominent-is-unique}).
\ensep
We begin by combining pairs of clauses of the form (\ref{eq:best-implies-tx-bis}),
which leads to the following ones:
\begin{equation}
\label{eq:good-nova1-bis}
\temp_x\, \lor\, \temp_y\,
\lor\, \bigvee_{\substack{z\neq x\\z\neq y}}\, p_{zx}\,
\lor\, \bigvee_{\substack{z\neq x\\z\neq y}}\, p_{zy},
\quad\text{for any two different $x,y\in\ist$.}
\end{equation}
One can now combine (\ref{eq:worst-implies-ntx-bis}) and (\ref{eq:tprominent-exists}). This results in
\begin{equation}
\label{eq:good-nova3}
\bigvee_{z\neq x}\, \temp_z\, \lor\,\, \bigvee_{z\neq x}\, p_{xz},\qquad\text{for any $x\in\ist$.}
\end{equation}
On the other hand, one can also combine clauses (\ref{eq:best-implies-tx-bis}) and (\ref{eq:tprominent-is-unique}), which gives
\begin{equation}
\label{eq:good-nova4}
\nt\temp_y\, \lor\, \bigvee_{z\neq x}\, p_{zx},\qquad \text{for any two different $x,y\in\ist$.}
\end{equation}
Notice that, in the special case of having only two options, (\ref{eq:good-nova1-bis}), (\ref{eq:good-nova3}) and (\ref{eq:good-nova4}) coincide respectively with (\ref{eq:tprominent-exists}), (\ref{eq:best-implies-tx-bis}) and (\ref{eq:worst-implies-ntx-bis}).
Finally, for more than two options one can combine (\ref{eq:good-nova4}) with itself, which leads to
\newcommand\xppr{x\rlap{$\scriptstyle'$}}
\begin{equation}
\label{eq:good-nova6}
\nt\temp_y\, \lor\, \bigvee_{\substack{z\neq x\\z\neq \xppr}}\, p_{zx}\, \lor\, \bigvee_{\substack{z\neq x\\z\neq \xppr}}\, p_{zx'},\qquad \text{for any three different $x,x'\!,y\in\ist$.}
\end{equation}

In~contrast to the doctrines that we have met so far,
here one cannot stay with disjoint resolution.
The problem lies in the derivation of (\ref{eq:good-nova6})$_{x,x'\!,y}$ from (\ref{eq:good-nova4})$_{x,y}$ and (\ref{eq:good-nova4})$_{x'\!,y}$.
Therefore, (\ref{eq:best-implies-tx-bis}--\ref{eq:tprominent-is-unique}) is \textbf{not} guaranteed to be $\ast$-equivalent to the corresponding Blake canonical form,
namely (\ref{eq:best-implies-tx-bis}--\ref{eq:good-nova6}).
In~such a situation, the standard course of action would be using the Blake canonical form.
\ensep
However, when looking at the rationale behind 
the clauses that have been obtained, one sees that they are
contained in
a more comprehensive doctrine 
that seems quite reasonable and
worth being adopted in its full generality.

This doctrine, that we will refer to as that of \dfc{comprehensive prominence,} 
is made up by the following clauses, two of which are indexed by arbitrary non-empty subsets of $\ist$:
\begin{alignat}{2}
\label{eq:rectangle-implies-good}
\bigvee_{r\in\xst}\temp_r\, &\lor\, \bigvee_{\substack{r\in\xst\\s\notin\xst}}\, p_{sr},\qquad
&&\text{for any non-empty $\xst\sbseteq\ist$;}
\\
\label{eq:rectangle-implies-bad}
\nt\temp_y\, &\lor\, \bigvee_{\substack{r\in\xst\\s\notin\xst}}\, p_{sr},\qquad
&&\text{for any non-empty $\xst\sbseteq\ist$, and any $y\notin\xst$;}
\\[3.5pt]
\label{eq:tprominent-is-unique-bis}
\nt\temp_x\, &\lor\, \nt\temp_y,\qquad &&\text{for any two different $x,y\in\ist$;}
\end{alignat}
Clauses (\ref{eq:rectangle-implies-good}) and (\ref{eq:rectangle-implies-bad}) are saying the following:
If there exists a non-empty $\xst\sbseteq\ist$ such that every $r\in\xst$ is preferred to any $s\notin\xst$,
then $\xst$ contains at least one prominent option, whereas $\ist\setminus\xst$ contains none.

One easily sees that clauses (\ref{eq:best-implies-tx-bis}), (\ref{eq:good-nova1-bis}), (\ref{eq:good-nova3}) and (\ref{eq:tprominent-exists}) are particular cases of (\ref{eq:rectangle-implies-good})
(in particular, the existence clause corresponds to the case $\xst=\ist$). On the other hand, (\ref{eq:worst-implies-ntx-bis}), (\ref{eq:good-nova4}) and (\ref{eq:good-nova6}) are particular cases of (\ref{eq:rectangle-implies-bad}). More specifically, the only difference between (\ref{eq:rectangle-implies-good}--\ref{eq:tprominent-is-unique-bis}) and (\ref{eq:best-implies-tx-bis}--\ref{eq:good-nova6}) is that the latter is restricted to subsets $\xst$ of size $|\xst|=1,2,N\!-\!1,N$, where $N=|\ist|$. Therefore, both doctrines are different from each other when $N\ge 5$. In~this case, (\ref{eq:rectangle-implies-good}--\ref{eq:tprominent-is-unique-bis}) contains clauses that cannot be derived from (\ref{eq:best-implies-tx-bis}--\ref{eq:tprominent-is-unique}).


The fact that (\ref{eq:rectangle-implies-good}) and (\ref{eq:rectangle-implies-bad}) are indexed by all possible subsets of $\ist$ makes things rather involved. However, we will see that
the results are interesting enough.

\medskip
The one-step revision transformation $v\mapsto\utv$ associated with (\ref{eq:rectangle-implies-good}--\ref{eq:tprominent-is-unique-bis})
can be written in the following form: For any $x,y\in\ist$,

\begin{alignat}{2}
\label{eq:one-step-tx-ct-prominence}
&\utv(\temp_x) &&= \max\Big(\, v(\temp_x),\,
\max_{\substack{\xst\sbseteq\ist\\\xst\ni x}}\,\,\min \big(\min_{\substack{r\in\xst\\r\neq x}} v(\nt \temp_r),\, \min_{\substack{r\in\xst\\s\notin\xst}} v(p_{rs}) \big)\,\Big),
\\
\label{eq:one-step-ntx-ct-prominence}
&\utv(\nt \temp_y) &&= \max\Big(\, v(\nt \temp_y),\,\, \max_{r\neq y} v(\temp_r),\,\,
\max_{\substack{\emptyset\neq\xst\sbseteq\ist\\\xst\not\ni y}}\,\min_{\substack{r\in\xst\\s\notin\xst}} v(p_{rs}) \,\Big),
\displaybreak[3]\\
\label{eq:one-step-pxy-ct-prominence}
&\utv(p_{yx}) &&= \max\bigg(v(p_{yx}),\,
\max_{\substack{\xst\sbseteq\ist\\\xst\ni x\\\xst\not\ni y}}\,
  \min\!\Big(\!\max\big(\min_{r\in\xst}v(\nt\temp_r), \max_{s\notin\xst}v(\temp_s)\big),\,\min_{\substack{r\in\xst\\s\notin\xst\\rs\neq{xy}}} v(p_{rs})\Big)
  \bigg),
\end{alignat}
where the operators $\max$ and $\min$
should be understood as giving respectively the values~$0$ and~$1$
whenever they are applied to an empty set.
\ensep
Recall that our aim is to iterate this transformation starting from an initial valuation satisfying~(\ref{eq:zero-suitable}).

\medskip
\begin{corollary}\hskip.5em
\label{st:comprehensive-Blake-cor}
The following equality holds whenever $\temp_x$ is accepted:
\begin{equation}
\label{eq:tunq}
\urv(\temp_x) \,=\, \utv(\temp_x) \,=\, \min_{s\neq x} v(p_{xs}).
\end{equation}
On the other hand, the following one holds for any $y$:
\begin{equation}
\label{eq:notunq}
\urv(\nt\temp_y) \,=\, \utv(\nt\temp_y) \,=\, \max_{\emptyset\neq\xst\sbseteq\ist\setminus\{y\}}\,\min_{\substack{r\in\xst\\s\notin\xst}}\, v(p_{rs}).
\end{equation}
\end{corollary}

\begin{proof}\hskip.5em
It follows from Proposition~\ref{st:comprehensive-Blake} on account of the formulas (\ref{eq:one-step-tx-ct-prominence}--\ref{eq:one-step-ntx-ct-prominence}) and the initial values (\ref{eq:zero-suitable}).
\end{proof}

\medskip
As in the supremacy doctrine,
the definite consistency theorem (Theorem~\ref{st:dec}) applied to (\ref{eq:tprominent-is-unique}) and (\ref{eq:tprominent-exists})
guarantees that:\ensep
(i) when $\temp_x$ is accepted, then $x$ is the only option with this property,\ensep
and (ii) when $\temp_x$ is not accepted for any $x\in\ist$,
then this proposition is undecided for more than one $x\in\ist$.\ensep
In~the sequel an option $x$ for which $\temp_x$ is not rejected 
will be called a \dfd{comprehensive prominence winner}. 

\medskip
\remark
Notice also that $x$ is the unique comprehensive prominence winner as soon as 
$\utv(\temp_x)>\utv(\nt\temp_x)$. In fact, starting from this inequality, the unquestionability of $\nt\temp_x$ and the fact that $\urv(\temp_x)\ge\utv(\temp_x)$ allow us to derive that $\urv(\temp_x)>\urv(\nt\temp_x)$.

\medskip
\begin{proposition}\hskip.5em
\label{st:cpw-criterion}
An option~$x$ is a comprehensive prominence winner \ifoi it minimizes $\urv(\nt\temp_x)$.
\end{proposition}

\begin{proof}\hskip.5em
Assume that $x$ is a comprehensive prominence winner. In order to see that it minimizes $\urv(\nt\temp_x)$ it suffices to notice that the following inequalities hold for any $y\neq x$:
\begin{equation}
\urv(\nt\temp_y) \,\ge\, \urv(\temp_x) \,\ge\, \urv(\nt\temp_x),
\end{equation}
where the first one holds because of the consistency of $\urv$ with the uniqueness clause (\ref{eq:tprominent-is-unique}).

\halfsmallskip
Assume now that $x$ minimizes $\urv(\nt\temp_x)$. In order to see that it is a comprehensive prominence winner, it suffices to notice that
\begin{equation}
\urv(\temp_x) \,\ge\, \min_{z\neq x}\urv(\nt\temp_z) \,\ge\, \min_{z}\urv(\nt\temp_z) \,=\, \urv(\nt\temp_x),
\end{equation}
where the first inequality holds because of the consistency of $\urv$ with the existence clause (\ref{eq:tprominent-exists}).
\end{proof}

\paragraph{6.4.2}
The comprehensive prominence doctrine has good properties in connection with majority-dominant sets.
A set $\smithset\sbseteq\ist$ is said to be \dfc{majority-dominant} when one has $\orv(p_{xy})>\onehalf$ for every $x\in\smithset$ and $y\notin\smithset$. 
If both $\smithset$ and $\smithsetbis$ are majority-dominant, then one must have either $\smithset\sbseteq\smithsetbis$ or $\smithsetbis\sbseteq\smithset$. Otherwise it would be incompatible with (\ref{eq:sum-less-than-one}). As a consequence, there is always a unique \dfc{minimal majority-dominant set $\goodset$.}
The case of a Condorcet winner is simply that where the minimal majority-dominant set consists of a single 
option.

The notion of minimal majority-dominant set was introduced by Benjamin Ward in 1961~\cite{ward},
and again by Irving John Good in 1971~\cite{good}.
This set is often called the Smith set (see for instance \cite[p.\,154]{t6}),
in reference to a subsequent work of John H.\ Smith \cite{smith};
however, the latter was not especially interested in locating the social winner in the minimal majority-dominant set, but only in looking at the social binary preferences associated with a general majority-dominant set.

\begin{proposition}\hskip.5em
\label{st:condorcet-good}
The minimal majority-dominant set $\goodset$ has the following properties:
\begin{alignat}{2}
\label{eq:ntymig}
&\urv(\nt\temp_y) \,&&>\,\textstyle{\onehalf},\quad  \text{for any $y\notin\goodset$,}
\\[2.5pt]
\label{eq:ntxmig}
&\urv(\nt\temp_x) \,&&\le\, \textstyle{\onehalf},\quad  \text{for any $x\in\goodset$,}
\\[2.5pt]
\label{eq:tymig}
&\urv(\temp_y) \,&&\le\, \textstyle{\onehalf},\quad  \text{for any $y\notin\goodset$.}
\end{alignat}
As a consequence, any comprehensive prominence winner is ensured to belong to $\goodset$.
\end{proposition}

\begin{proof}\hskip.5em
The inequality (\ref{eq:ntymig}) follows readily from (\ref{eq:notunq}) by taking $\xst=\goodset$.

In view of (\ref{eq:notunq}), in order to prove (\ref{eq:ntxmig}) we must show that
\begin{equation}
\label{eq:minrectangle}
\min_{\substack{r\in\xst\\s\notin\xst}}\, v(p_{rs}) \,\le\, \textstyle{\onehalf},\quad 
\text{whenever $x\in\goodset$ and $\emptyset\neq\xst\sbseteq\ist\setminus\{x\}$.}
\end{equation}
For $\xst\not\sbset\goodset$, this inequality, and even the corresponding strict one, is ensured to hold because there exists $z\in\xst\setminus\goodset$ such that  $v(p_{xz}) > 1/2$ and therefore $v(p_{zx})< 1/2$.\ensep
On the other hand, for $\xst\sbset\goodset$, the minimality of $\goodset$ guarantees the existence of $z\in\xst$ and $y\in\goodset\setminus\xst$ such that $v(p_{zy})\le 1/2$, which also ensures (\ref{eq:minrectangle}).

The inequality (\ref{eq:tymig}) follows easily from (\ref{eq:ntxmig}) because of the
consistency of $\urv$ with the clause (\ref{eq:tprominent-is-unique-bis}).

Finally, (\ref{eq:ntymig}) together with (\ref{eq:tymig}) says that $\temp_y$ is rejected for any $y\notin\goodset$. Therefore, the comprehensive prominence winner(s) must belong to $\goodset$.
\end{proof}


\begin{corollary}
\label{st:condorcet-good-corollary}
The comprehensive prominence method complies with the Condorcet principle (in its majority version \textup{M1}).
\end{corollary}

\remark In the case of $x$ being a Condorcet winner, the strict inclusion $\xst\sbset\goodset=\{x\}$ is not possible for a non-empty set. Therefore, (\ref{eq:ntxmig}) and (\ref{eq:tymig}) hold then with strict inequality, since this is what is obtained in the proof of Proposition~\ref{st:condorcet-good} for $\xst\not\sbset\goodset$.

\medskip
\begin{proposition}\hskip.5em
\label{st:condorcet-smith}
If $\smithset\sbseteq\ist$ is a majority-dominant set, then
\begin{equation}
\label{eq:pyxmig}
\urv(p_{yx}) \,\le\,\textstyle{\onehalf},\quad \text{whenever $x\in\smithset$ and $y\notin\smithset$.}
\end{equation}
As a consequence, the comprehensive prominence method accepts $p_{xy}$ for any $x\in\smithset$ and $y\notin\smithset$, and this decision is unquestionable.
\end{proposition}

\begin{proof}\hskip.5em
%
%
In order to establish (\ref{eq:pyxmig}), we will base ourselves on equation (\ref{eq:one-step-pxy-ct-prominence}) and Lem\-ma~\ref{st:lem-best-option}, which allow us to write
$$
\urv(p_{yx}) = \max\!\bigg(\!v(p_{yx}),\,
\max_{\substack{\xst\sbseteq\ist\\\xst\ni x\\\xst\not\ni y}}\,
\min\!\Big(\!\max\big(\min_{r\in\xst}\urv(\nt\temp_r),
\max_{s\notin\xst}\urv(\temp_s)\big),\min_{\substack{r\in\xst\\s\notin\xst\\rs\neq{xy}}} \urv(p_{rs})\Big)
\!\bigg).
$$
Now, the consistency of $\urv$ with the clause (\ref{eq:tprominent-is-unique-bis}) implies that $\urv(\temp_s) \le \urv(\nt\temp_r)$ for any $s\neq r$. Therefore, 
\begin{equation}
\label{eq:cond-smith-inter}
\urv(p_{yx}) \,=\, \max\Big(\, v(p_{yx}),\,
\max_{\substack{\xst\sbseteq\ist\\\xst\ni x\\\xst\not\ni y}}\,
\min\!\big(\min_{r\in\xst}\urv(\nt\temp_r),\min_{\substack{r\in\xst\\s\notin\xst\\rs\neq{xy}}} \urv(p_{rs})\big)
\Big).
\end{equation}
Let us assume that $x\in\smithset$ and $y\notin\smithset$.
Since we know that $\orv(p_{yx})<\onehalf$, (\ref{eq:pyxmig})~will be established if we are able to show that
\begin{equation}
\label{eq:dilemma}
\min\!\big(\min_{r\in\xst}\urv(\nt\temp_r),\min_{\substack{r\in\xst\\s\notin\xst\\rs\neq{xy}}} \urv(p_{rs})\big) \,\le\, \textstyle{\onehalf},\quad 
\text{whenever $\xst\ni x$ and $\xst\not\ni y.$}
\end{equation}
In order to obtain this property, we will distinguish two possibilities 
depending on whether or not $\xst$ intersects the minimal majority-dominant set $\goodset$ considered in the preceding proposition.
\ensep
Let us begin by considering the case $\xst\cap\goodset\neq\emptyset$.
In this case, (\ref{eq:dilemma}) holds because of (\ref{eq:ntxmig}).\ensep
In the special case that $x\in\goodset$, this argument covers all of the sets $\xst$ considered in (\ref{eq:cond-smith-inter}--\ref{eq:dilemma}) (because $\xst$ is restricted to contain $x$).
Therefore, (\ref{eq:pyxmig}) is by now established for $\smithset=\goodset$.
This fact allows us to fix the pending case $\xst\cap\goodset=\emptyset$. Indeed, in this case the already obtained result (for $\smithset=\goodset$) guarantees that $\urv(p_{xs})\le\onehalf$ for every $s\in\goodset$, which values are included in the left-hand side of (\ref{eq:dilemma}) (since $y\notin\smithset$ implies $xs\neq xy$).

Since $\urv(p_{xy}) \ge \orv(p_{xy}) > \onehalf$, having obtained (\ref{eq:pyxmig}) ensures that $p_{xy}$ is accepted for any $x\in\smithset$ and $y\notin\smithset$. On the other hand, since we have not only $\urv(p_{xy}) > \urv(p_{yx})$, but even $\orv(p_{xy}) > \urv(p_{yx})$, we can be sure that this decision does not rely on belief derived from unsatisfiable conjunctions.
\end{proof}

\begin{comment}\textbf{Conjectura} (interessaria per a les decisions amb un marge):\hskip.5em
En les condicions de la proposició precedent potser es pot demostrar que $p_{xy}$ és inqüestionable, és a dir que $\urv(p_{xy})=\utv(p_{xy})$ per a $x\in\smithset$ i $y\notin\smithset$. La demostració usaria el Teorema~4.8 de \cite{dp} (que més avall apareix amb la numeració 6.3). Una idea onírica seria la següent: L'evidència a favor de $p_{xy}$ només pot venir de $\clau=(\ref{eq:rectangle-implies-good})$ o de $\clau=(\ref{eq:rectangle-implies-bad})$. En ambdós casos $\clau$ conté tot un rectangle $\rect_\xst$ amb $\xst\not\ni x$. Si $\xst\ne\smithset$ (el conjunt de l'enunciat) sembla que les $\nt\lxt$ de la condició (b) del Teorema~3.14 inclourien alguna $p_{ba}$ amb $a\in\smithset$ i $b\notin\smithset$, per a la qual sabem que $\urv(p_{ba})\le\onehalf$, la qual cosa donaria un mínim $\le\onehalf<\urv(p_{xy})$. Si $\xst=\smithset$, llavors ens haurien de salvar les $\nt\temp_w$ ($w\in\goodset$) i/o les $\temp_z$ ($z\notin\goodset$)%
\end{comment}

\medskip
\remark The above-remarked fact that (\ref{eq:ntxmig}) holds as a strict inequality whenever $x$ is the Condorcet winner entails that \textup{(\ref{eq:pyxmig})} holds also as a strict inequality whenever $x$ is the Condorcet winner and $y\neq x$.

\paragraph{6.4.3} 
A bit unexpectedly,
the comprehensive prominence winner is often the same as the maximin one (\secpar{6.1}):
\begin{proposition}\hskip.5em
\label{st:compremaximin}
Whenever there is a unique comprehensive prominence winner,
then there is also a unique maxi\-min winner,
and they coincide with each other.
\end{proposition}
\begin{proof}\hskip.5em
As it has been remarked, $x$ being the unique comprehensive prominence winner is equivalent to say that $\temp_x$ is accepted, \ie $\urv(\temp_x)>\urv(\nt\temp_x)$. According to Cor.\,\ref{st:comprehensive-Blake-cor}, this translates into the first of the next two inequalities:
\begin{equation}
\min_{s\neq x} v(p_{xs}) \,\,>\,
\max_{\xst\sbseteq\ist\setminus\{x\}}\,\min_{\substack{r\in\xst\\s\notin\xst}}\, v(p_{rs})
\,\,\ge\,\, \min_{s\neq y} v(p_{ys}).
\end{equation}
The second of these inequalities is easily seen to hold for any $y\neq x$: it suffices to consider $\xst=\{y\}$ in the central expression. Therefore, we get $\minrow_x>\minrow_y$ for any $y\neq x$, \,\ie $x$ is the unique maximin winner.
\end{proof}

\medskip
However, the converse is not true:
It may happen that there is a unique maximin winner
but there is not a unique comprehensive prominence winner.
Not only that, a~unique maximin winner may even be rejected as a prominent option.
Example:
1~$a\cdsucc b\cdsucc c\cdsucc d$,
1~$a\cdsucc b\cdsucc d\cdsucc c$,
2~$b\cdsucc c\cdsucc a\cdsucc d$,
1~$b\cdsucc c\cdsucc d\cdsucc a$,
1~$c\cdsucc a\cdsucc d\cdsucc b$,
1~$d\cdsucc a\cdsucc b\cdsucc c$,
2~$d\cdsucc c\cdsucc a\cdsucc b$.
The Llull matrix is
\begin{equation}
\label{eq:cf-maximin}
\begin{tabular}{|c|c|c|c|}
\hline\rule{0pt}{2.5ex}%
\diaglabel{a}& 6 & 3 & 5\\
\hline\rule{0pt}{2.5ex}%
3 & \diaglabel{b} & 6 & 5\\
\hline\rule{0pt}{2.5ex}%
6 & 3 & \diaglabel{c} & 5\\
\hline\rule{0pt}{2.5ex}%
4 & 4 & 4 & \diaglabel{d}\\
\hline
\end{tabular}\,,
\end{equation}
which shows that the maximin winner $d$ is defeated by any other option!
The possibility of such situations has been pointed out as the main drawback of 
the maximin method \cite[p.\,212--213]{t6}.
In contrast, the comprehensive prominence method definitely rejects $d$ as a prominent choice,
and it remains undecided between the other three options
(more specifically, one gets $\Urv(\temp_x) - \Urv(\nt\temp_x) = 4 - 4 = 0$ for $x=a,b,c$
and $\Urv(\temp_d) - \Urv(\nt\temp_d) = 4 - 5 = -1$).
This agrees with Proposition~\ref{st:condorcet-good} since here the minimal majority-dominant set is clearly $\goodset=\{a,b,c\}$.

\bigskip
Let us assume that $0.75$ of the first vote of the preceding example changes
from $a\cdsucc b\cdsucc c\cdsucc d$ to $a\cdsucc b\cdsucc d\cdsucc c$.
The Llull matrix becomes then
\begin{equation}
\label{eq:cf-maximin-bis}
\begin{tabular}{|c|c|c|c|}
\hline\rule{0pt}{2.5ex}%
\diaglabel{a}& 6 & 3 & 5\\
\hline\rule{0pt}{2.5ex}%
3 & \diaglabel{b} & 6 & 5\\
\hline\rule{0pt}{2.5ex}%
6 & 3 & \diaglabel{c} & 4.25\\
\hline\rule{0pt}{2.5ex}%
4 & 4 & 4.75 & \diaglabel{d}\\
\hline
\end{tabular}\,.
\end{equation}
From here, the comprehensive prominence method results in
$\Urv(\temp_x) - \Urv(\nt\temp_x) = 4 - 4 = 0$ for $x=a,b,c$
and $\Urv(\temp_d) - \Urv(\nt\temp_d) = 4 - 4.25 = -0.25$.
In this case, the minimal majority-dominant set is not $\{a,b,c\}$
but the whole of $\ist=\{a,b,c,d\}$.
Even so, however, the set of comprehensive prominence winners still reduces to $\{a,b,c\}$.
As it will be shown in the next result,
this has to do with the fact that this set has the property of maximizing the quantity $\sigma_\xst = \min_{x\in\xst,\,y\notin\xst}\, \orv(p_{xy})$. From now on, a proper subset of $\ist$ with this property will be called a \dfc{maximin set.}

\begin{proposition}\hskip.5em
\label{st:compremaximin-bis}
If the set of comprehensive prominence winners is not the whole of $\ist$, then it is contained in the intersection of all the maximin sets.
\end{proposition}
\begin{proof}\hskip.5em
It suffices to show that the following implication holds for any maxi\-min set:
If $z\notin\xst$ then $z$ is not a comprehensive prominence winner.
\ensep
This is a consequence of Corollary~\ref{st:comprehensive-Blake-cor} and Proposition~\ref{st:cpw-criterion}.
In fact, since $\xst$ maximizes $\sigma_\xst = \min_{x\in\xst,y\notin\xst}\, \orv(p_{xy})$, (\ref{eq:notunq}) allows to derive that any $z\notin\xst$ satisfies $\urv(\nt\temp_z)\ge\urv(\nt\temp_x)$ for any $x\neq z$.
On the other hand, if $z$ were a comprehensive prominence winner, then Proposition~\ref{st:cpw-criterion} ensures that $\urv(\nt\temp_z)<\urv(\nt\temp_x)$ for any $x$ that is not such a winner,
thus obtaining a contradiction.
\end{proof}

\noindent
In the way that we have defined it in \secpar{6.4.1}, the set of comprehensive prominence winners is never empty. Therefore, the preceding proposition has the following consequence:

\begin{corollary}\hskip.5em
\label{st:compremaximin-ter}
It the maximin sets have an empty intersection, then the comprehensive prominence winner is undecided between the whole of $\ist$.
\end{corollary}



\paragraph{6.4.4} 
Let us perturb example~(\ref{eq:cf-maximin}) so that the Condorcet cycle $a\cdsucc b\cdsucc c\cdsucc a$ becomes uneven.
The Llull matrix could take, for instance, the following value:
\renewcommand\diaglabel[1]{\cellcolor[gray]{0.8}\makebox[2.5em][c]{$#1$}}
\begin{equation}
\label{eq:cf-maximin-ter}
\begin{tabular}{|c|c|c|c|}
\hline\rule{0pt}{2.5ex}%
\diaglabel{a}& 6 & $3\!+\!2\epsilon$ & 5\\
\hline\rule{0pt}{2.5ex}%
3 & \diaglabel{b} & $6\!-\!\epsilon$ & 5\\
\hline\rule{0pt}{2.5ex}%
$6\!-\!2\epsilon$ & $3\!+\!\epsilon$ & \diaglabel{c} & 5\\
\hline\rule{0pt}{2.5ex}%
4 & 4 & 4 & \diaglabel{d}\\
\hline
\end{tabular}\,,
\end{equation}
with $\epsilon>0$. A bit unexpectedly, the comprehensive prominence method does not select $a$ as the unique winner (which corresponds to breaking the cycle by the weakest link): for $0\le\epsilon<\onehalf$, 
\ie when the victories within the cycle are stronger than those outside it,
the result is still an undecidedness between $a,b,c$
(with the same values of $\Urv(\temp_x)$ and $\Urv(\nt\temp_x)$ as for $\epsilon=0$).

%

In this example one gets $\Urv(\temp_x)=\Utv'(\temp_x)=4$ for $x=a,b,c$. These values derive from unsatisfiable conjunctions. In fact, they derive through the concatenation of $\temp_x \leftarrow \bigwedge_{r\ne x}\,\nt\temp_r$ with
$\nt\temp_r \leftarrow \bigwedge_{s\ne d}\,p_{ds}$ for $r=a,b,c$
as~well~as
$\nt\temp_d \leftarrow \bigwedge_{s\ne d}\,p_{sd}$.
This concatenation amounts to the implication $\temp_x \leftarrow \bigwedge_{s\ne d}\,(p_{sd}\land p_{ds})$, whose right-hand side negates the \textit{tertium non datur} clauses $p_{sd}\lor p_{ds}$. 
So the undecidedness between $a,b,c$ as comprehensive prominence winners is questionable.
According to Proposition~\ref{st:comprehensive-Blake}, what is unquestionable is the rejection of~$d$.

This suggests that in the event of undecidedness one should restrict the attention to all the undecided winners and start again the comprehensive prominence algorithm from the corresponding restriction of the original Llull matrix.
Generally speaking, this progressive elimination could involve several rounds.
We will refer to this procedure as the \textbf{refined comprehensive prominence method}. 
\ensep
In the preceding example, this procedure selects $a$ as a single winner.

\paragraph{6.4.5} 
In the incomplete case, one easily finds examples where the comprehensive prominence winner 
does not coincide with the transitivity one.
For instance: 
1~$a\cdsucc b\cdsucc c$, 
1~$b\cdsucc c\cdsucc a$, 
2~$c\cdsucc a\cdsucc b$, 
1~$a$, 
2~$b$;
the ranking produced by the transitivity doctrine is $a\cdsucc b\cdsucc c$,
whereas the comprehensive prominence winner is~$b$.

In contrast, in the complete case, there is a strong experimental evidence that the transitivity winners are always included among the refined comprehensive prominence winners.
On the other hand, one easily finds examples where this inclusion is strict.
A~proof of the stated inclusion is lacking.
A~weaker fact, proved by Schulze \cite[\secpar{4.8}]{sc}, is the following: In the complete case, the transitivity winners are included in the union of all the maximin sets (which union Schulze calls the MinMax set). When there is only one transitivity winner, then Schulze's proof is easily adapted to show that this winner is contained in the intersection of all maximin sets.

\section{Goodness}

Generally speaking, 
$x$~being preferred to any other option does not imply $x$~being good:
in fact, $x$ could be the lesser of several evils.
In other words, what we have called supremacy does not imply goodness;
as a consequence, since supremacy does imply prominence,
the latter cannot either be identified with goodness.
\ensep
In fact, goodness has an absolute character:
it does not rely on comparing an option to another,
but it makes sense for every option by itself.
\ensep
On the other hand, preference is related to goodness in the following way:
if $x$ is considered good and $y$ is considered bad, then $x$ is preferred to~$y$.

So, the doctrine that relates goodness to preference 
is made of the following clauses, where $\good_x$ denotes the proposition `$x$~is good':
$\good_x \land \nt\good_y \rightarrow  p_{xy}$.
\ensep
In~conjunctive normal form:
\begin{equation}
\label{eq:goodness}
\nt\good_x \,\lor\, p_{xy}\, \lor\, \good_y,\qquad \text{for any two different $x,y\in\ist$;}
\end{equation}
This doctrine will be referred to as the \dfd{goodness doctrine}. Clearly, it is symmetric under negation.

Notice that, similarly to \secpar{6.2}, here we are admitting the possibility of having several good options
as well as having none of them.
To~the effect of making a choice, it makes sense to select the option(s) $x$
for which the proposition $\good_x$ gets a highest acceptability,
\ie a highest value of the difference $\urv(\good_x)-\urv(\nt\good_x)$.
Notice however, that this number could be negative for all options,
which means that no good option is found;
in this case we are choosing the lesser of the evils.
On the other hand, it can also happen that we find several good options.
In this case, choosing the one(s) with highest acceptability corresponds to deciding by a margin (\secpar{2.2}).

We will refer to this procedure as the \dfc{goodness method,} and to its winners as the \dfc{goodness winners.}

\medskip
The one-step revision transformation $v\mapsto\utv$ associated with the clauses \textup{(\ref{eq:goodness})} takes the following form:
For any $x,y\in\ist$:
\begin{alignat}{2}
\label{eq:one-step-good}
&\utv(\good_x) \,&&=\, \max\Big(\,
v(\good_x),\, \max_{y\neq x}\,\min\big(v(\good_y),v(p_{xy})\big) \,\Big),
\\
\label{eq:one-step-not-good}
&\utv(\nt \good_x) \,&&=\, \max\Big(\,
v(\nt \good_x),\, \max_{y\neq x}\,\min\big(v(\nt\good_y),v(p_{yx})\big) \,\Big),
\\
\label{eq:one-step-pxyfromgood}
&\utv(p_{xy}) \,&&=\, \max\Big(\, v(p_{xy}),\,
\min\big(v(\good_x),v(\nt \good_y)\big)
\,\Big).
\end{alignat}


\medskip
The properties of disjoint resolvability and unquestionability (Prop.\,\ref{st:goodness-Blake-unquestionable}) allow us to express $\urv$ directly as the result of the one-step transformation associated with the Blake canonical form, namely:
\begin{alignat}{2}
\label{eq:urv-good}
&\urv(\good_x) \,&&=\, \,\,\Max_{x_0=x}\,\, \min\big(\,
\orv(p_{x_0x_1}), \orv(p_{x_1x_2}), \dots, \orv(p_{x_{n-1}x_n}), \orv(\good_{x_n})
\,\big),
\\[3.5pt]
\label{eq:urv-not-good}
&\urv(\nt \good_y) \,&&=\, \,\,\Max_{x_n=y}\,\, \min\big(\,
\orv(\nt\good_{x_0}), \orv(p_{x_0x_1}), \orv(p_{x_1x_2}), \dots, \orv(p_{x_{n-1}x_n})
\,\big),
\\
\label{eq:urv-pxyfromgood}
&\urv(p_{xy}) \,&&=\,  \,\,\max\Big(\, \orv(p_{xy}),\,
\min\big(\urv(\good_x),\urv(\nt \good_y)\big)
\,\Big),
\end{alignat}
where the $\Max$ operators of (\ref{eq:urv-good}--\ref{eq:urv-not-good}) consider paths
$x_0x_1\dots x_n$ of length $n\ge 0$
with all $x_i$ pairwise different (notice that $n=0$ corresponds to the \textit{tertium non datur} clauses $\good_x\lor\nt\good_x$ and $\good_y\lor\nt\good_y$).

\paragraph{7.1}{\textbf{Approval-disapproval voting.}}

\medskip
In approval voting each voter is asked for a list of approved options \cite{brams}.
Clearly, the fraction of voters who approve a given option~$x$ can be seen 
as the collective degree of belief in the goodness of~$x$, \ie as the value of $\orv(\good_x)$.
The standard approval-voting rule for choosing an option $x$ is simply
taking the one that maximizes~$\orv(\good_x)$.

Now, from the point of view of this article,
in order to make a decision about $\good_x$
we should consider also the support for $\nt\good_x$.
Properly speaking, this requires that voters specifically pronounce themselves about it.
This idea is considered in \cite{felsenthal}, whose CAV~rule chooses the option that maximizes
the difference~$\orv(\good_x) - \orv(\nt\good_x)$.

In this section we assume that the votes contain no direct information about binary preferences.
This amounts to having $\orv(p_{xy})=0$ for any $x$ and~$y$. In this case,
(\ref{eq:urv-good}--\ref{eq:urv-pxyfromgood}) reduce to 
$\urv(\good_x) = \orv(\good_x),$ $\urv(\nt\good_x) = \orv(\nt\good_x)$ and $\urv(p_{xy}) =  \min(\orv(\good_x),\orv(\nt\good_y))$.
\ensep
Therefore, the acceptability of $\good_x$ is simply the difference
$\orv(\good_x) - \orv(\nt\good_x)$.
So \textit{the goodness method fully coincides in this case with the CAV~rule}.

This rule coincides with that of standard approval voting, \ie rating the options by $\orv(\good_x)$, in the following two cases: (i)~$\orv(\nt\good_x) = 0$, (ii)~$\orv(\nt\good_x) = 1-\orv(\good_x)$, which correspond respectively to interpreting that (i)~non-ap\-proved options are not necessarily disapproved, or contrarily,
that (ii)~all non-approved options are disapproved. 
\ensep
In the case of interpretation~(ii) an option is accepted as a good one \ifoi $\orv(\good_x)>\onehalf$, \ie if it is approved by a majority.

\paragraph{7.2}{\textbf{Approval-disapproval-preferential voting.}}

\medskip
Let us consider now the case where the individual votes give information not only about approval or disapproval, but also about binary preferences that are not a consequence of approval and disapproval. For instance, besides saying that $x$ and $y$ are both approved (or both disapproved), a voter can add the information that he prefers $y$ to $x$.
This added information may lead to the need for revising the degrees of belief about the goodness or badness of the different options. For instance, it might happen that $x$ is approved more often than $y$ but at the same time $y$ is preferred to $x$ more often than $x$ is preferred to $y$. Such was indeed the case in \secpar{1.2}.

With more or less generality, such forms of voting have been considered by several authors 
(see \cite[ch.\,3]{brams}, \cite{glmmp10}).
A real example is the 2006 Public Choice Society election~\cite{pcs06},
whose actual ballots are listed in \cite[§3.3]{cri}.

The PAV procedure proposed in \cite{bramssanver} and \cite[\secpar{3.3}]{brams}
gives priority to the approval information, which decides the winner unless several candidates are approved by a majority; in this case, the attention is restricted to the set $\majap$ of these majority-approved candidates, and the preferential information about them is used to single out, if possible, their Condorcet winner; if this is not possible, then the attention is restricted to the minimal majority dominant subset of~$\majap$
---\ie the smallest subset $\majapgoodset$ of $\majap$ with the property that each $x\in\majapgoodset$ preferred to any $y\in\majap\setminus\majapgoodset$--- \,\,and the winner is selected from~$\majapgoodset$ by looking again at the approval score.

Though certainly reasonable, this procedure alternates between approval and preferential information
in a categorical way that does not seem fully justified.
\ensep
In particular, it is not difficult to set up examples where a tiny preference margin between two majority-approved candidates may select a~candidate approved by a small majority instead of another one that was approved by a very large majority.
Consider, for instance, the following profile:
\begin{equation}
\label{eq:profileepsilon}
\textstyle{\onehalf}+\varepsilon\,:\,\; a\better b\fiav\,,\quad
\textstyle{\onehalf}-\varepsilon\,:\,\; b\fiav a,
\end{equation}
where the parameter $\varepsilon$ is assumed to be positive but quite small. One easily checks that $\orv(\good_a) = \onehalf+\varepsilon$, $\orv(\good_b) = 1$, and $\orv(p_{ab}) = \onehalf+\varepsilon$. Therefore, the PAV procedure gives the victory to $a$ on the basis of a nearly vanishing margin of preference, in spite of the fact that the approval of $b$ is unanimous whereas that of $a$ is near to only half the vote.
\ensep


In contrast, our method carefully gauges the interplay between both kinds of information in accordance with the doctrine under consideration.
\ensep
The goodness doctrine that we are considering in this section contains only the clauses~(\ref{eq:goodness})
and their derivates~(\ref{eq:goodnesschain}). More particularly, it does not
include the transitivity of preferences.
Having said that, it is interesting to see that (\ref{eq:goodnesschain}) shows that
having a chain of preferences from $x_n$ to $x_0$, \ie having $p_{x_nx_{n-1}} \land ... \land p_{x_1x_0} \land p_{x_1x_0}$, implies $\good_{x_n}\lor\nt\good_{x_0}$ just as well as the direct preference~$p_{x_nx_0}$.

\medskip
In accordance with the doctrine (\ref{eq:goodness}), our proposal to deal with combined approval and preference information is to iterate the transformation (\ref{eq:one-step-good}--\ref{eq:one-step-pxyfromgood}) until invariance, and then select the option~$x$ with a highest value of $\urv(\good_x)-\urv(\nt\good_x)$, which we have already called the goodness winner.

\medskip
For the profile (\ref{eq:profileepsilon}) one gets $\urv(\good_a)-\urv(\nt\good_a) = 2\varepsilon$ and $\urv(\good_b)-\urv(\nt\good_b) = \onehalf+\varepsilon$;
therefore, for small values of $\varepsilon$ the goodness winner is~$b$.
So the goodness method does not let a slight margin of preference to prevail over a big difference in approval.
\ensep
In other cases, however, preferences can overturn an initial difference in approval. Such a phenomenon occurs for instance in the following example:
\begin{equation}
\label{eq:profiletwo}
5: a\fiav b\better c,\quad
4: b\better c\fiav a,\quad
3: c\fiav a\better b,\quad
1: a\better c\fiav b.
\end{equation}
The values of $\Orv(\good_x)-\Orv(\nt\good_x)$ for $x=a,b,c$ are respectively $-1,-5,3$;
so~initially ---without taking into account the preferential information--- the only approved candidate is $c$.
After revision, however, the values of $\Urv(\good_x)-\Urv(\nt\good_x)$ are respectively $1,-1,-1$;
so the final decision rejects the goodness/approval
of $c$ and chooses $a$ as the only approved candidate.


\bigskip
The following result establishes a highly desirable property of monotonicity.
Its hypotheses are similar to those of Theorem~\ref{st:mono} but the conclusion is stronger.

\begin{theorem} 
\label{st:apdv-monotonicity}
When applied to approval-disapproval-preferential voting, the
goodness method is monotonic in the following sense:
Assume that some votes are modified by raising $x$ to a better position (with no other change).
In this case, the acceptability of $\good_x$ either increases or stays constant. 
Furthermore,
if this acceptability is~initially larger than (or equal to) that of $\good_y$,
then it remains so after raising $x$. That is, if
\begin{equation}
\label{eq:monotonia-goodness}
 \urv(\good_x)-\urv(\nt\good_x) \,\ge\,  \urv(\good_y)-\urv(\nt\good_y)    
\end{equation}
for some $y$, then the same inequality holds after raising~$x$ in the votes.
And if the inequality is strict, then it remains so.
\end{theorem}
\begin{proof} 
Raising~$x$ decomposes into several cases: 
(i)~$\orv(\good_x)$ increases; 
(ii)~$\orv(\nt\good_x)$ decreases;\ensep
(iii)~$\orv(p_{xz})$ increases;\ensep
(iv)~$\orv(p_{zx})$ decreases.

From (\ref{eq:urv-good}) and (\ref{eq:urv-not-good}) we see that, in cases (i) and (iii), $\urv(\good_x)$ can increase, while $\urv(\nt\good_x)$ cannot. On the contrary, 
in cases (ii) and (iv), $\urv(\nt\good_x)$ can decrease while $\urv(\good_x)$ cannot. This gives the first part of the statement. Moreover, it follows that the inequality (\ref{eq:monotonia-goodness}) is preserved ---as well as its strict counterpart--- if both $\urv(\good_y)$ and $\urv(\nt\good_y)$ stay constant. So \emph{it only remains to deal with the case where either $\urv(\good_y)$ or $\urv(\nt\good_y)$ (or both) change. From now on we assume that this is the case}.

Inequality (\ref{eq:monotonia-goodness}) can be rewritten as
\begin{equation}
\label{eq:monotonia-goodness-2}
\urv(\good_x)-\urv(\good_y) \,\ge\, \urv(\nt\good_x) -\urv(\nt\good_y).
\end{equation}
Although a change in favour of $x$ can also benefit $y$,
we will see, case by case, that one of the sides of (\ref{eq:monotonia-goodness-2}) 
clearly stays constant or changes in the right direction,
whereas the other stays constantly equal to zero.
 
\smallskip

In connection with formulas (\ref{eq:urv-good}) and (\ref{eq:urv-not-good}) it will be convenient to use the following notation. Given a path $\gamma=x_0x_1\dots x_n$, we write
$$v_\gamma\,:=\,\min(\orv(p_{x_0x_1}),\orv(p_{x_1x_2}),\dots,\orv(p_{x_{n-1}x_n})),$$
including $v_\gamma=1$ if $\gamma$ is an empty path (i.e. $n=0$).
We understand also that our paths do not have repeated elements.
With this notation, (\ref{eq:urv-good}) and (\ref{eq:urv-not-good}) take respectively the following form:
\begin{alignat}{2}
\label{eq:urv-good-gamma}
&\urv(\good_s) \,&&=\, \,\max_{\substack{t\\ \gamma\,:\,s\rightarrow t}}\min(v_\gamma, \orv(\good_t)),
\\[3.5pt]
\label{eq:urv-not-good-gamma}
&\urv(\nt\good_s) \,&&=\, \,\max_{\substack{t\\ \gamma\,:\,t\rightarrow s}}\min(\orv(\nt\good_t), v_\gamma).
\end{alignat}

\medskip\noindent
\emph{Case~(i):~$\orv(\good_x)$ increases.}\hskip.5em
By (\ref{eq:urv-not-good-gamma}) it is clear that both $\urv(\nt\good_x)$ and $\urv(\nt\good_y)$ stay constant. So, under our assumptions,  $\urv(\good_y)$ does not stay constant.   
However, we will prove that it stays equal to $\urv(\good_x),$
which ensures that inequality (\ref{eq:monotonia-goodness-2}) is preserved.

In view of (\ref{eq:urv-good-gamma}),
the only way for $\urv(\good_y)$ to vary with $\orv(\good_x)$ is that the maximum of the right-hand side of (\ref{eq:urv-good-gamma}) for $s = y$ be realized by $t = x$. More specifically, one must have
\begin{equation}
 \label{eq:monotonia-desigualtat-0-demo}
\urv(\good_y)
\,=\,\max_{\substack{\gamma\,:\,y\rightarrow x}} \min(v_\gamma, \orv(\good_x))
\,=\,\min(M,\orv(\good_x))
\,=\,\orv(\good_x) \,<\, M,
\end{equation}
where we have set $M = \displaystyle\max_{\gamma\,:\,y\rightarrow x}v_\gamma$.
In fact, having $\orv(\good_x) \ge M$ would result in $\urv(\good_y)$
staying equal to~$M$.
Obviously, (\ref{eq:monotonia-desigualtat-0-demo}) will hold as long as $\orv(\good_x)$ remains smaller than~$M$ (when $\orv(\good_x)$ goes past this value we fall into the already settled case where both $\urv(\good_y)$ and $\urv(\nt\good_y)$ stay constant).

Using Theorem~\ref{st:char} and part~(c) of Theorem~\ref{st:rev}, is also clear from (\ref{eq:urv-good-gamma}) that 
\begin{equation}
 \label{eq:monotonia-desigualtat-00-demo}
\urv(\good_y)
\,\ge\, \max_{\substack{\gamma\,:\,y\rightarrow x}} \min (v_\gamma,\urv(\good_x))
\,=\, \min(M,\urv(\good_x))
\,=\, \urv(\good_x),
\end{equation}
the last equality being true because otherwise we would get $\urv(\good_y)\ge M$,
in contradiction with (\ref{eq:monotonia-desigualtat-0-demo}).

By combining these facts with the inequality $\urv(\good_x)\ge\orv(\good_x)$, we get
$$\orv(\good_x)=\urv(\good_y)\ge\urv(\good_x)\ge\orv(\good_x),$$
which gives the claimed equality, namely $\urv(\good_y)=\urv(\good_x)$.
As it has been already mentioned, all of this holds
as long as $\orv(\good_x)$ remains smaller than~$M$,
and past this value we fall into the already settled case
where both $\urv(\good_y)$ and $\urv(\nt\good_y)$ stay constant.

\medskip\noindent
\emph{Case~(iii): $\orv(p_{xz})$ increases for some $z\ne x$.}\hskip.5em
From (\ref{eq:urv-not-good-gamma}), $\urv(\nt\good_x)$ stays constant.
If~$\urv(\nt\good_y)$ increases, then the right-hand side of (\ref{eq:monotonia-goodness-2}) decreases, so it changes in the right direction.
If $\urv(\good_y)$ stays constant, then the left-hand side of (\ref{eq:monotonia-goodness-2}) changes also in the right direction. 
So it remains to deal with the case where $\urv(\good_y)$ does not stay constant.
We will prove that in this case $\urv(\good_y)$ stays equal to $\urv(\good_x),$
which ensures that inequality (\ref{eq:monotonia-goodness-2}) is preserved.

Similarly to case~(i), the hypothesis that $\urv(\good_y)$ varies with $\orv(p_{xz})$ implies that
\begin{multline}
 \label{eq:monotonia-desigualtat-1-demo}
\urv(\good_y)\,=\,
\max_{\substack{t\ne x,y\\[1pt]\gamma\,:\,y\rightarrow x\\\delta:\,z\rightarrow t}}
\,
\min\left(v_\gamma, \orv(p_{xz}),v_\delta,\orv(\good_t)\right) \\ 
\,=\,\min(M, \orv(p_{xz}),N)
\,=\,\orv(p_{xz})
\,<\, M,\,N,\quad
\end{multline}
where we have set 
$M=\displaystyle\max_{\gamma\,:\,y\rightarrow x} v_\gamma$
and 
$N=\displaystyle\max_{\substack{t\ne x,y\\ \delta:\,z\rightarrow t }} \min(v_\delta, \orv(\good_t)).$ 
This holds as long as $\orv(p_{xz})$ remains smaller than $\min(M,N)$, 
after which value we fall into the case where $\urv(\good_y)$ stays constant.

As before, the inequality $\urv(\good_y)\ge\urv(\good_x)$ is ensured because of (\ref{eq:monotonia-desigualtat-00-demo}), whose last equality holds now because otherwise we would be in contradiction with (\ref{eq:monotonia-desigualtat-1-demo}).

On the other hand, we can also write
\begin{multline}
\label{eq:monotonia-desigualtat-11-demo}
\urv(\good_x)
\,\ge\,\max_{\substack{t\ne x\\ \delta\,:\,z\rightarrow t}} \min( \orv(p_{xz}),v_\delta, \orv(\good_t))
\,\ge\,\max_{\substack{t\ne x,y\\ \delta\,:\,z\rightarrow t}} \min( \orv(p_{xz}),\min(v_\delta, \orv(\good_t)))\\
\,=\,\min( \orv(p_{xz}),N)
\,=\,\orv(p_{xz}),
\end{multline}
where the last equality must hold because otherwise we would get $\urv(\good_x) \ge N$,
in contradiction with the already known facts that $\urv(\good_x)\le \urv(\good_y)<N$.

So our claim that $\urv(\good_y)$ stays equal to $\urv(\good_x)$
is ensured by the following chain of inequalities:
$$\orv(p_{xz})=\urv(\good_y)\ge\urv(\good_x)\ge\orv(p_{xz}),$$
which hold as long as $\orv(p_{xz})$ remains smaller than $\min(M,N)$.
After this value we fall into the already settled case where $\urv(\good_y)$ stays constant.

\medskip\noindent
\emph{Cases~(ii):~$\orv(\nt\good_x)$ decreases, 
\,and\,~(iv):~$\orv(p_{zx})$ decreases for some $z\ne x$.}\hskip.5em
In view of the symmetry of formulas (\ref{eq:urv-good}) and (\ref{eq:urv-not-good})
it is clear that these cases are respectively analogous to~(i) and~(iii).
So they are omitted.
\end{proof}

\medskip
\begin{corollary}
\label{st:apdv-monotonicity-cor}
Consider the case of approval-disapproval-preferential voting. 
Assume that some votes are modified by raising $x$ to a better position (with no other change).
If $x$ was initially a goodness winner, then it remains so.
\end{corollary}

\section{Recapitulation and concluding remarks}

We have been looking at the problem of collectively choosing between a finite number of options.
The traditional view about it focuses on binary preferences
and the fact that the majority criterion often does away with transitivity.
In contrast, here we have taken the view 
that the problem is not only about binary preferences,
\ie whether an option is preferred to another, 
but also about some notion of choiceness,
\ie whether an option is considered a right choice, 
and about which constraints are imposed on these notions.
As in the traditional view, the majority criterion will often fail at preserving the assumed constraints.


When several individual opinions are aggregated, 
every issue, \ie whether an option deserves being chosen or whether it is preferable to another,
becomes valued by the fraction of people 
who support it.
As we have been doing, it makes sense to think of these numbers as degrees of collective belief
that need being revised so as to achieve consistency with the required constraints.

More specifically, we have restricted ourselves to a revision method that we introduced in \cite{dp} and that uses only the max and min operators. This automatically excludes
such reputed rules as that of Borda and that of Condorcet, Kem\'eny and Young \cite{mu,nitzan,t6}.
Even so, we still obtain a variety of known rules as well as some new ones.


Which rule is obtained depends on which constraints are assumed.
In other words, it depends on which notion of choiceness is considered.
A common view in this connection ``reduces'' 
choosing to determining a complete ordering:
the~right choice is the option that goes first in the right complete ordering.
As we have seen, this notion of choiceness leads to the so-called method of paths.

Removing transitivity leads to other notions of choiceness,
namely suprem\-acy and prom\-in\-ence,
that are related respectively to the plurality rule and the Condorcet principle.

\medskip
However, none of the preceding notions of choiceness takes into account
whether an option is really good or not.
Of course, a good option is preferable to a bad one.
But an option being preferable to another does not imply 
the former being good nor the latter being bad.

A single individual might have no better possibility than choosing the lesser of several evils.
However, for a group of people it is not that simple.
As we saw in \secpar{1.2}, the collective preferences may be really at odds
with the collective approval information.
More specifically, the majority criterion may give a complete ordering,
but it may well happen that the best option according to this ordering
is at the same time the most disapproved one.

In order to resolve such undesirable inconsistencies
there is no other solution than first revealing them,
that is, asking the individuals
for both kinds of information, preferences and approval,
and then correcting them
by means of an appropriate method.

A method being appropriate means having good properties.
In this connection, we have seen that our method discussed in \secpar{7}
is monotonic in the sense of Theorem~\ref{st:apdv-monotonicity}
and its Corollary~\ref{st:apdv-monotonicity-cor}.
This adds to the general properties of \secpar{2},
such as respect for consistent majority decisions (Theorem~\ref{st:majority})
and respect for unanimity (Theorem~\ref{st:unanimity}).

\medskip
Although the goodness doctrine does not include transitivity, in \secpar{7} we saw that different options are still compared through paths that involve other options, as in the transitivity doctrine.
This suggests that a doctrine combining both goodness and transitivity could have also good properties. Another subject for future work is the application to these doctrines of other methods that have recently been proposed in the general context of judgment aggregation \cite{duddypiggins2012,dietrich2013}.

\appendix
\normalsize 

\section{Appendix: Technical proofs}

In order to make sure that the upper revised valuations that we are dealing with are the right ones,
we must use either the Blake canonical form or a $\ast$-equivalent one.
According to \cite[Cor.\,4.6]{dp}, in order to prove $\ast$-equivalence
it suffices to check that the conjunctive normal form under consideration
is disjoint-resolvable as defined in \cite[\secpar{4.2}]{dp}.

On the other hand,
in order to make sure that the revised degrees of belief do not derive from unsatisfiable conjunctions,
we are interested in the properties of unquestionability that we mentioned in~\secpar{2.5}.
In order to obtain such properties, it suffices to check for the sufficient conditions that are given in
\cite[Thm.\,4.8, Cor.\,4.9]{dp}.


In this appendix we briefly outline these verifications
for the doctrines considered in this article.
They are tedious and rather mechanical.
It would be most appropriate to be able to entrust this work to some symbolic programming tool.
Unfortunately, however, such a tool is not yet available to us.


\vskip2ex 
\medskip
\begin{proposition}\hskip.5em
\label{st:tech-trans}
The transitivity conjunctive normal form \textup{(\ref{eq:transitivity})} is disjoint-resolvable
and therefore $\ast$-equivalent to the corresponding Blake canonical form.
This doctrine is unquestionable for every $p_{xy}$.
\end{proposition}

\begin{proof}[Scheme of the proof]\hskip.5em
The Blake canonical form, with the \textit{tertium non datur} clauses included, consists of all clauses of the form
\begin{equation}
\label{eq:transitivitychain}
p_{x_0x_1} \lor\, p_{x_1x_2} \lor\, \dots \,\lor\, p_{x_{n-1}x_n} \lor\, p_{x_nx_0},
\end{equation}
with $n\ge1$ and all $x_i$ $(0\le i\le n)$ pairwise different (which restricts $n$ to be less than or equal to the number of elements of $\ist$).
These clauses are easily derived from (\ref{eq:transitivity}) by successive concatenation, and suitably ordered this derivation will use only disjoint resolution.

The unquestionability is easily obtained through condition~(a) of\linebreak 
\cite[Thm.\,4.8]{dp}, which amounts to the following fact: if two cycles without repetitions contain opposed links, then suppressing both of these links and putting together all the others results in a set $L$ of links with the following property: for any link contained in $L$ there exists a cycle without repetitions that is included in $L$ and contains that link. The reader will easily convince himself ---maybe by means of some drawings--- that this is really a fact.
\end{proof}

\medskip
\begin{proposition} 
\label{st:tech-sprm1}
The supremacy conjunctive normal form \hbox{\textup{(\ref{eq:beats-all-implies-best}--\ref{eq:best-exists})}}
is dis\-joint-re\-solv\-able and therefore $\ast$-equivalent to the corresponding Blake canonical form. 
\end{proposition}

\begin{proof}[Scheme of the proof]\hskip.5em
We will limit ourselves to indicating that one can arrive at the Blake canonical form through disjoint resolution
by means of the following procedure:\ensep
First, each clause of the form (\ref{eq:best-implies-beats-all}) is combined by disjoint resolution 
with the clause of the same form where $x$ and $y$ are interchanged with each other.
This produces the clauses 
\begin{equation}
\label{eq:best-is-unique-bis}
\hbox to0pt{\hss\small$(\ref{eq:best-is-unique})$\hskip18mm} 
\nt\sprm_x\, \lor\,\, \nt\sprm_y,\qquad \text{for any two different $x,y\in\ist$.}
\end{equation}
Second, one successively applies disjoint resolution to combine clause (\ref{eq:best-exists}) with one or more clauses of the form (\ref{eq:best-implies-beats-all}), each of them corresponding to a different $x$ and admitting any $y\neq x$. This produces all clauses of the form
\begin{equation}
\label{eq:best-nova}
\bigvee_{x\in\xst}\, \sprm_x\, \lor\,\, \bigvee_{x\in\ist\setminus\xst}\, p_{xf(x)},
\end{equation}
where $\xst$ is any subset of $\ist$, and $f$ is any mapping from $\ist\setminus\xst$ to $\ist$ with $f(x) \neq x$.
The interested reader can go over the rather tedious task of checking that no further resolution is possible. 
\end{proof}

\medskip
\begin{proposition}\hskip.5em
\label{st:supremacy-is-unquestionable}%
The supremacy doctrine is unquestionable for $\nt\sprm_x$, \ie it satisfies $\urv(\nt\sprm_x)=\utv(\nt\sprm_x)$, and it is also unquestionable for $\sprm_x$ when accepted, \ie it satisfies $\urv(\sprm_x)=\utv(\sprm_x)$ whenever $\urv(\sprm_x)>\urv(\nt\sprm_x)$.
\end{proposition}

\begin{proof}[Scheme of the proof]\hskip.5em
The unquestionability for $\nt\sprm_x$ is obtained by checking that condition~(a) of \cite[Thm.\,4.8]{dp} is satisfied for any pair of clauses $\clau,\clau'$ and any literal $\liit$ satisfying $\liit,\nt\sprm_x\in\clau$, $\nt\liit\in\clau'$ and $\liit\neq\nt\sprm_x$. Finally, the unquestionability for $\sprm_x$ when accepted is obtained by checking that either condition~(a) or condition~(b$'$) of \linebreak[3] 
\cite[Cor.\,4.9]{dp} is satisfied in the analogous situation for $\sprm_x$ instead of $\nt\sprm_x$.
\end{proof}

\medskip
\begin{proposition}\hskip.5em
The conjunctive normal form formed by the clauses 
\textup{(\ref{eq:best-implies-tx})} is disjoint-resolvable
and therefore $\ast$-equivalent to the corresponding Blake canonical form.
\end{proposition}

\begin{proof}[Scheme of the proof]\hskip.5em
The only possibility for producing new clauses is combining pairs of clauses of the form (\ref{eq:best-implies-tx}),
which results in the following ones:
\begin{equation}
\label{eq:good-nova1}
\temp_x\, \lor\, \temp_y\,
\lor\, \bigvee_{\substack{z\neq x\\z\neq y}}\, p_{zx}\,
\lor\, \bigvee_{\substack{z\neq x\\z\neq y}}\, p_{zy},
\quad\text{for any two different $x,y\in\ist$.}
\end{equation}
One can check that no further resolution is possible. 
\end{proof}

\medskip
\begin{proposition}\hskip.5em
The symmetric prominence conjunctive normal form 
\textup{(\ref{eq:best-implies-tx}--\ref{eq:worst-implies-ntx})} is disjoint-resolvable
and therefore $\ast$-equivalent to the corresponding Blake canonical form.
\end{proposition}

\begin{proof}[Scheme of the proof]\hskip.5em
Once again, we will limit ourselves to indicating a path that allows to arrive at the Blake canonical form through disjoint resolution.
First, we combine pairs of clauses of the form (\ref{eq:best-implies-tx}), which results in (\ref{eq:good-nova1}).
Second, we combine pairs of clauses of the form (\ref{eq:worst-implies-ntx}) to obtain
\begin{equation}
\label{eq:good-nova2}
\nt\temp_x\, \lor\, \nt\temp_y\,
\lor\, \bigvee_{\substack{z\neq x\\z\neq y}}\, p_{xz}\,
\lor\, \bigvee_{\substack{z\neq x\\z\neq y}}\, p_{yz},
\quad\text{for any two different $x,y\in\ist$.}
\end{equation}
One can check that no further resolution is possible. 
\end{proof}

\medskip
\begin{proposition}
\label{st:comprehensive-Blake}
The comprehensive prominence conjunctive normal form 
\textup{(\ref{eq:rectangle-implies-good}--\ref{eq:tprominent-is-unique-bis})}
is already the Blake canonical form.
This doctrine is unquestionable for $\nt\temp_x$, \ie it satisfies $\urv(\nt\temp_x)=\utv(\nt\temp_x)$,
and it is also unquestionable for $\temp_x$ when accepted, \ie it satisfies $\urv(\temp_x)=\utv(\temp_x)$ whenever $\urv(\temp_x)>\urv(\nt\temp_x)$.
\end{proposition}

\begin{proof}[Scheme of the proof] 
The first statement requires checking that all the would-be resolutions are absorbed by some clause already present.
The unquestionability for $\nt\temp_x$ and the unquestionability for $\temp_x$ when accepted are obtained respectively from 
\cite[Thm.\,4.8]{dp} and \cite[Cor.\,4.9]{dp} by checking for their respective conditions, as in Prop.\,\ref{st:supremacy-is-unquestionable}.
\end{proof}

\medskip
\begin{proposition}
\label{st:goodness-Blake-unquestionable}
The goodness conjunctive normal form, formed by the clauses 
\textup{(\ref{eq:goodness})}, is disjoint-resolvable
and therefore $\ast$-equivalent to the corresponding Blake canonical form.
This doctrine is unquestionable for any of the propositions $\good_x$ and~ $\nt\good_x$.
\end{proposition}

\begin{proof}[Scheme of the proof]\hskip.5em
The Blake canonical form consists of all clauses of the form
\begin{equation}
\label{eq:goodnesschain}
\nt\good_{x_0} \lor\, p_{x_0x_1} \lor\, p_{x_1x_2} \lor\, \dots \,\lor\, p_{x_{n-1}x_n} \lor\, \good_{x_n},
\end{equation}
with $n\ge1$ and all $x_i$ $(0\le i\le n)$ pairwise different (which restricts $n$ to be less than or equal to the number of elements of $\ist$).
The clauses (\ref{eq:goodnesschain}) are easily derived from (\ref{eq:goodness}) by successive concatenation, and suitably ordered this derivation will use only disjoint resolution.

The unquestionability $\good_x$ and $\nt\good_x$ is easily obtained through condition~(a) of\linebreak[3] 
\cite[Thm.\,4.8]{dp}. For $\good_x$ ---the case of $\nt\good_x$ is analogous by symmetry--- this condition requires the following: for any clause $\clau$ of the form (\ref{eq:goodnesschain}) that includes $\good_x$ (\ie $x_n=x$), and any other clause $\clau'$ of the form (\ref{eq:goodnesschain}), if $\clau$ and $\clau'$ contain respectively $\liit$ and $\nt\liit$, then the would-be resolution $\clau\res{\liit}\clau'$ contains a third clause $\clau_1$ of the form (\ref{eq:goodnesschain}) that still includes $\good_x$. The reader will easily become convinced that it is so, both in the case where $\liit=\nt\good_y$ ($y=x_0$) and in the case where $\liit=p_{ab}$.
\end{proof}

\begin{comment}
Not unquestionable for $p_{xy}$? Condition (a) is not satisfied.
\end{comment}

\end{document}